\documentclass{article}

\usepackage{graphicx} 
\usepackage{subfigure}
\usepackage{wrapfig}
\usepackage{xspace}
\usepackage{caption}

\usepackage{algorithm}
\usepackage{algorithmic}

\usepackage{enumitem}
\usepackage{hyperref}

\usepackage{fullpage}
\usepackage{authblk}
\usepackage[round]{natbib} 
\renewcommand{\cite}{\citep}

\usepackage{Definitions}

\usepackage{color}

\renewcommand{\le}{\leqslant}

\renewcommand{\ge}{\geqslant}

\begin{document}

\title{Stochastic Generative Hashing}

\author{
    $^*$Bo Dai$^{1}$, \footnote{Authors are equally contributed.} Ruiqi Guo$^{2}$, Sanjiv Kumar$^2$, Niao He$^3$, Le Song$^1$\\
    $^1$ Georgia Institute of Technology\\
    bodai@gatech.edu, lsong@cc.gatech.edu\\
    $^2$ Google Research, NYC\\
    \{guorq, sanjivk\}@google.com\\
    $^3$ University of Illinois at Urbana-Champaign\\
    niaohe@illinois.edu
}

\maketitle

\begin{abstract}
	Learning-based binary hashing has become a powerful paradigm for fast search and retrieval in massive databases. However, due to the requirement of discrete outputs for the hash functions, learning such functions is known to be very challenging. In addition, the objective functions adopted by existing hashing techniques are mostly chosen heuristically. In this paper, we propose a novel generative approach to learn hash functions through Minimum Description Length principle such that the learned hash codes maximally compress the dataset and can also be used to regenerate the inputs. We also develop an efficient learning algorithm based on the stochastic distributional gradient, which avoids the notorious difficulty caused by binary output constraints, to jointly optimize the parameters of the hash function and the associated generative model. Extensive experiments on a variety of large-scale datasets show that the proposed method achieves better retrieval results than the existing state-of-the-art methods. 
\end{abstract}


\section{Introduction}\label{sec:intro}

\setlength{\abovedisplayskip}{3pt}
\setlength{\abovedisplayshortskip}{3pt}
\setlength{\belowdisplayskip}{3pt}
\setlength{\belowdisplayshortskip}{3pt}
\setlength{\jot}{2pt}

\setlength{\floatsep}{2ex}
\setlength{\textfloatsep}{2ex}

Search for similar items in web-scale datasets is a fundamental step in a number of applications, especially in image and document retrieval. Formally, given a reference dataset $X = \{x_i\}_{i=1}^N$ with $x\in \Xcal\subset \RR^d$, we want to retrieve similar items from $X$ for a given query $y$ according to some similarity measure $sim(x, y)$. When the negative Euclidean distance is used, \ie,  $sim(x, y) = -\|x-y\|_2$, this corresponds to $L_2$ Nearest Neighbor Search (L2NNS) problem; when the inner product is used, \ie, $sim(x, y) = x^\top y$, it becomes a Maximum Inner Product Search~(MIPS) problem. In this work, we focus on L2NNS for simplicity, however our method handles MIPS problems as well, as shown in the supplementary material~\ref{appendix:sgh_mips}. Brute-force linear search is expensive for large datasets. To alleviate the time and storage bottlenecks, two research directions have been studied extensively: (1) partition the dataset so that only a subset of data points is searched; (2) represent the data as codes so that similarity computation can be carried out more efficiently. The former often resorts to search-tree or bucket-based lookup; while the latter relies on binary hashing or quantization. These two groups of techniques are orthogonal and are typically employed together in practice. 

In this work, we focus on speeding up search via binary hashing. Hashing for similarity search was popularized by influential works such as Locality Sensitive Hashing~\cite{lsh,gionis1999similarity,charikar2002similarity}. The crux of binary hashing is to utilize a hash function, $f(\cdot) : \Xcal\rightarrow \{0, 1\}^{l}$, which maps the original samples in $\Xcal\in \RR^d$ to $l$-bit binary vectors $h \in \{0, 1\}^l$ while preserving the similarity measure, \eg, Euclidean distance or inner product. Search with such binary representations can be efficiently conducted using Hamming distance computation, which is supported via \emph{POPCNT} on modern CPUs and GPUs. Quantization based techniques~\cite{babenko2014additive, pq, cq} have been shown to give stronger empirical results but tend to be less efficient than Hamming search over binary codes~\cite{douze2016polysemous,kmeanshashing}.

Data-dependent hash functions are well-known to perform better than randomized ones~\cite{hashingsurvey}. Learning hash functions or binary codes has been discussed in several papers, including spectral hashing~\cite{spectralhashing}, semi-supervised hashing~\cite{WanKumCha10}, iterative quantization~\cite{itq}, and  others~\cite{anchorgraph,bilinearhashing,circulanthashing,shen2015learning, guo2015quantization}. The main idea behind these works is to optimize some objective function that captures the preferred properties of the hash function in a supervised or unsupervised fashion. 

Even though these methods have shown promising performance in several applications, they suffer from two main drawbacks: (1) the objective functions are often heuristically constructed without a principled characterization of goodness of hash codes, and (2) when optimizing, the binary constraints are crudely handled through some relaxation, leading to inferior results~\cite{liu2014discrete}. In this work, we introduce Stochastic Generative Hashing (SGH) to address these two key issues. We propose a generative model which captures both the encoding of binary codes $h$ from input $x$ and the decoding of input $x$ from $h$. This provides a principled hash learning framework, where the hash function is learned by Minimum Description Length (MDL) principle. Therefore, its generated codes can compress the dataset maximally. Such a generative model also enables us to optimize distributions over discrete hash codes without the necessity to handle discrete variables. Furthermore, we introduce a novel distributional stochastic gradient descent method which exploits distributional derivatives and generates higher quality hash codes. Prior work on binary autoencoders~\cite{CarRaz15} also takes a generative view of hashing but still uses relaxation of binary constraints when optimizing the parameters, leading to inferior performance as shown in the experiment section. We also show that binary autoencoders can be seen as a special case of our formulation. In this work, we mainly focus on the unsupervised setting\footnote{The proposed algorithm can be extended to supervised/semi-supervised setting easily as described in the supplementary material~\ref{appendix:generalization}.}.

\section{Stochastic Generative Hashing (SGH)}\label{sec:stoc_generative}

We start by first formalizing the two key issues that motivate the development of the proposed algorithm.

\textbf{Generative view}. Given an input $x\in \RR^d$, most hashing works in the literature emphasize modeling the forward process of \emph{generating binary codes from input}, \ie, $h(x) \in \{0, 1\}^l$, to ensure that the generated hash codes preserve the local neighborhood structure in the original space. Few works focus on modeling the reverse process of \emph{generating input from binary codes}, so that the reconstructed input has small reconstruction error. In fact, the generative view provides a natural learning objective for hashing. Following this intuition, we model the process of generating $x$ from $h$, $p(x|h)$ and derive the corresponding hash function $q(h|x)$ from the generative process. Our approach is not tied to any specific choice of $p(x|h)$ but can adapt to any generative model appropriate for the domain. In this work,  we show that even using a simple generative model (Section~\ref{sec:generative}) already achieves the state-of-the-art performance. 

\textbf{Binary constraints}. The other issue arises from dealing with binary constraints. One popular approach is to relax the constraints from $\{0, 1\}$~\cite{spectralhashing}, but this often leads to a large optimality gap between the relaxed and non-relaxed objectives.  Another approach is to enforce the model parameterization to have a particular structure so that when applying alternating optimization, the algorithm can alternate between updating the parameters and binarization efficiently. For example, \cite{itq, angularhashing} imposed an orthogonality constraint on the projection matrix, while \cite{circulanthashing} proposed to use circulant constraints, and \cite{ZhaZhaLiGuo14} introduced Kronecker Product structure. Although such constraints alleviate the difficulty with optimization, they substantially reduce the model flexibility. In contrast, we avoid such constraints and propose to optimize the \emph{distributions} over the binary variables to avoid directly working with binary variables. This is attained by resorting to the stochastic neuron reparametrization (Section~\ref{sec:learn}), which allows us to back-propagate through the layers of weights using the stochsastic gradient estimator.

Unlike~\cite{CarRaz15} which relies on solving expensive integer programs, our model is end-to-end trainable using \emph{distributional stochastic gradient descent} (Section~\ref{sec:dist_grad}). Our algorithm requires no iterative steps unlike iterative quantization (ITQ) \cite{itq}. The training procedure is much more efficient with guaranteed convergence compared to alternating optimization for ITQ. 

In the following sections, we first introduce the generative hashing model $p(x|h)$ in Section~\ref{sec:generative}. Then, we describe the corresponding process of generating hash codes given input $x$, $q(h|x)$ in Section~\ref{sec:hash}. Finally, we describe the training procedure based on the Minimum Description Length (MDL) principle and the stochastic neuron reparametrization in Sections~\ref{sec:training} and~\ref{sec:learn}. We also introduce the distributional stochastic gradient descent algorithm in Section~\ref{sec:dist_grad}. 

\subsection{Generative Model $p(x|h)$}\label{sec:generative}
Unlike most works which start with the hash function $h(x)$, we first introduce a generative model that defines the likelihood of generating input $x$ given its binary code $h$, \ie, $p(x|h)$. It is also referred as a \textit{decoding} function. The corresponding hash codes are derived from an encoding function $q(h|x)$, described in Section~\ref{sec:hash}.

We use a simple Gaussian distribution to model the generation of $x$ given $h$:
\begin{equation}
p(x, h) \!=\! p(x|h)p(h), \textrm{ where} ~ p(x|h) \!=\! \Ncal(Uh,\!\rho^2I) \label{eqn:generative}
\end{equation}
and $U=\{u_i\}_{i=1}^l,\,  u_i\in \RR^d$ is a codebook with  $l$ codewords. The prior $p(h) \sim \Bcal(\theta) = \prod_{i=1}^l \theta_i^{h_i}(1 - \theta_i)^{1-h_i}$ is modeled as the multivariate Bernoulli distribution on the hash codes, where $\theta = [\theta_i]_{i=1}^l \in [0, 1]^l$. Intuitively, this is an additive model which reconstructs $x$ by summing the selected columns of $U$ given $h$, with a Bernoulli prior on the distribution of hash codes. The joint distribution can be written as:
\begin{eqnarray}\label{eqn:reduced_mrf}
p(x, h) \propto \exp\rbr{\frac{1}{2\rho^2}\underbrace{\rbr{x^\top x + h^\top U^\top Uh - 2x^\top Uh}}_{\|x - U^\top h\|_2^2} -  {(\log\frac{\theta}{1 - \theta})}^\top h }
\end{eqnarray}
This generative model can be seen as a restricted form of general Markov Random Fields in the sense that the parameters for modeling correlation between latent variables $h$ and correlation between $x$ and $h$ are shared. However, it is more flexible compared to Gaussian Restricted Boltzmann machines~\cite{Krizhevsky09,RanHin10} due to an extra quadratic term for modeling correlation between latent variables. We first show that this generative model preserves local neighborhood structure of the $x$ when the Frobenius norm of $U$ is bounded.
\begin{proposition}\label{prop:ann_preserve}
If $\|U\|_F$ is bounded, then the Gaussian reconstruction error, $\| x - U h_x \|_2$ is a surrogate for Euclidean neighborhood preservation.
\end{proposition}
\begin{proof}
Given two points $x, y \in \RR^d$, their Euclidean distance is bounded by
\begin{eqnarray*}
&&\|x - y\|_2 \\
&=&\|(x - U^\top h_x) - (y - U^\top h_y) + (U^\top h_x - U^\top h_y)\|_2\\
&\le& \|x -U^\top h_x\|_2 + \|y - U^\top h_y\|_2 + \|U^\top (h_x - h_y)\|_2 \\
&\le& \|x - U^\top h_x\|_2 + \|y - U^\top h_y\|_2 + \|U\|_F\|h_x - h_y\|_2 
\end{eqnarray*}
where $h_x$ and $h_y$ denote the binary latent variables corresponding to $x$ and $y$, respectively. Therefore, we have
$$
\|x - y\|_2 -  \|U\|_F\|h_x - h_y\|_2 \le \|x - U^\top h_x\|_2 + \|y - U^\top h_y\|_2
$$
which means minimizing the Gaussian reconstruction error, \ie, $-\log p(x|h)$, will lead to Euclidean neighborhood preservation. 
\end{proof}
A similar argument can be made with respect to MIPS neighborhood preservation as shown in the supplementary material~\ref{appendix:sgh_mips}. Note that the choice of $p(x|h)$ is not unique, and any generative model that leads to neighborhood preservation can be used here. In fact, one can even use more sophisticated models with multiple layers and nonlinear functions. In our experiments, we find complex generative models tend to perform similarly to the Gaussian model on datasets such as \texttt{SIFT-1M} and \texttt{GIST-1M}. Therefore, we use the Gaussian model for simplicity.

\subsection{Encoding Model $q(h|x)$}\label{sec:hash}
Even with the simple Gaussian model~(\ref{eqn:generative}), computing the posterior $p(h|x) = \frac{p(x, h)}{p(x)}$ is not tractable, and finding the MAP solution of the posterior involves solving an expensive integer programming subproblem. Inspired by the recent work on variational auto-encoder~\cite{KinWel13,MniGre14,GreDanMniBlu14}, we propose to bypass these difficulties by parameterizing the encoding function as
\begin{eqnarray}\label{eqn:encoder}
q(h|x) = \prod_{k=1}^l q(h_k=1|x)^{h_k} q(h_k=0|x)^{1 - h_k},
\end{eqnarray}
to approximate the exact posterior $p(h|x)$. With the linear parametrization, $h = [h_k]_{k=1}^l \sim \Bcal(\sigma(W^\top x))$ with $W = [w_k]_{k=1}^l$. At the training step, a hash code is obtained by sampling from $\Bcal(\sigma(W^\top x))$. At the inference step, it is still possible to sample  $h$. More directly, the MAP solution of the encoding function (\ref{eqn:encoder}) is readily given by
\vspace{-3mm}
\[
h(x) = \argmax_{h} q(h|x) = \frac{\operatorname{sign}(W^\top x) + 1}{2}
\]
This involves only a linear projection followed by a sign operation, which is common in the hashing literature. Computing $h(x)$ in our model thus has the same amount of computation as ITQ~\cite{itq}, except without the orthogonality constraints. 


\subsection{Training Objective}\label{sec:training}

Since our goal is to reconstruct $x$ using the least information in binary codes, we train the variational auto-encoder using the Minimal Description Length (MDL) principle, which finds the best parameters that maximally compress the training data. The MDL principle seeks to minimize the expected amount of information to communicate $x$:
$$L(x) = \sum_h q(h|x) (L(h) + L(x|h))$$
where $L(h)=-\log p(h) + \log q(h|x)$ is the description length of the hashed representation $h$ and $L(x|h)=-\log p(x|h)$ is the description length of $x$ having already communicated $h$ in~\cite{HinVan93, HinZem94, MniGre14}. By summing over all training examples $x$, we obtain the following training objective, which we wish to minimize with respect to the parameters of $p(x|h)$ and $q(h|x)$:
\begin{equation}\label{eqn:helmholtz}
\min_{\Theta=\{W, U, \beta, \rho\}} H(\Theta) := \sum_x L(x; \Theta) = -\sum_{x} \sum_{h} q(h|x) (\log p(x, h) - \log q(h|x)),
\end{equation}
where $U, \rho$ and $\beta := \log\frac{\theta}{1 - \theta}$ are parameters of the generative model $p(x, h)$ as defined in (\ref{eqn:generative}), and $W$ comes from the encoding function $q(h|x)$ defined in (\ref{eqn:encoder}). This objective is sometimes called Helmholtz (variational) free energy~\cite{Williams80, Zellner88,DaiHeDaiSon16}. When the true posterior $p(h|x)$ falls into the family of (\ref{eqn:encoder}), $q(h|x)$ becomes the true posterior $p(h|x)$, which leads to the shortest description length to represent $x$. 

We emphasize that this objective no longer includes binary variables $h$ as parameters and therefore avoids optimizing with discrete variables directly. This paves the way for continuous optimization methods such as stochastic gradient descent (SGD) to be applied in training. As far as we are aware, this is the first time such a procedure has been used in the problem of unsupervised learning to hash. Our methodology serves as a viable alternative to the relaxation-based approaches commonly used in the past.

\subsection{Reparametrization via Stochastic Neuron}\label{sec:learn}

Using the training objective of (\ref{eqn:helmholtz}), we can directly compute the gradients w.r.t. parameters of $p(x|h)$. However, we cannot compute the stochastic gradients w.r.t. $W$ because it depends on the stochastic binary variables $h$. In order to back-propagate through stochastic nodes of $h$, two possible solutions have been proposed. First, the reparametrization trick~\cite{KinWel13} which works by introducing auxiliary noise variables in the model. However, it is  difficult to apply when the stochastic variables are discrete, as is the case for $h$ in our model. On the other hand, the gradient estimators based on REINFORCE trick~\cite{BenLeoCou13} suffer from high variance. Although some variance reduction remedies have been proposed~\cite{MniGre14,GuLevSutMni15}, they are either biased or require complicated extra computation in practice. 

In next section, we first provide an \emph{unbiased} estimator of the gradient w.r.t. $W$ derived based on distributional derivative, and then, we derive a \emph{simple} and \emph{efficient} approximator. Before we derive the estimator, we first introduce the {stochastic neuron} for reparametrizing Bernoulli distribution. A stochastic neuron reparameterizes each Bernoulli variable $h_k(z)$ with $z \in (0, 1)$. Introducing random variables $\xi \sim \Ucal(0,1)$, the stochastic neuron is defined as
\begin{eqnarray} \label{eqn:doubly_sn}
\tilde{h}(z, \xi):= \begin{cases}
    1       & \quad \text{if } z\ge \xi  \\
    0  & \quad \text{if } z < \xi \\
\end{cases}.
\end{eqnarray}
Because $\PP(\tilde{h}(z,\xi)=1)=z$, we have $\tilde{h}(z,\xi)\sim \Bcal(z)$. We use the stochastic neuron~\eq{eqn:doubly_sn} to reparameterize our binary variables $h$ by replacing $[h_k]_{k=1}^l(x) \sim \Bcal(\sigma(w_k^\top x))$ with $[\tilde{h}_k(\sigma(w_k^\top x), \xi_k)]_{k=1}^l$. Note that $\tilde{h}$ now behaves deterministically given $\xi$. This gives us the reparameterized version of our original training objective (\ref{eqn:helmholtz}):
\begin{align}\label{eqn:reparam_helmholtz}
\tilde{H}(\Theta) = \sum_{x}\Htil(\Theta; x) := \sum_{x} \EE_{\xi}\sbr{\ell(\htil, x)},
\end{align}
where $\ell(\htil, x) :=  -\log p(x, \tilde{h}(\sigma(W^\top x), \xi)) + \log q(\tilde{h}(\sigma(W^\top  x), \xi) |x)$ with $\xi\sim \Ucal(0, 1)$. With such a reformulation, the new objective can now be optimized by exploiting the distributional stochastic gradient descent, which will be explained in the next section.

\section{Distributional Stochastic Gradient Descent}\label{sec:dist_grad}

For the objective in~\eq{eqn:reparam_helmholtz}, given a point $x$ randomly sampled from $\{x_i\}_{i=1}^N$, the stochastic gradient $\widehat\nabla_{U, \beta, \rho}\Htil(\Theta; x)$ can be easily computed in the standard way. However, with the reparameterization, the function $\Htil(\Theta; x)$ is no longer differentiable with respect to $W$ due to the discontinuity of the stochastic neuron  $\htil(z, \xi)$. Namely, the SGD algorithm is not readily applicable. To overcome this difficulty, we will adopt the notion of \emph{distributional derivative} for generalized functions or distributions~\cite{Grubb08}. 

\subsection{Distributional derivative of Stochastic Neuron}

Let $\Omega\subset\RR^d$ be an open set. Denote $\Ccal_0^\infty(\Omega)$ as the space of the functions that are infinitely differentiable with compact support in $\Omega$. Let $\Dcal'(\Omega)$ be the space of continuous linear functionals on $\Ccal_0^\infty(\Omega)$, which can be considered as the dual space. The elements in space  $\Dcal'(\Omega)$ are often called general {\sl distributions}. We emphasize this definition of distributions is more general than that of traditional probability distributions.
%
\begin{algorithm}[t!]
\caption{\textbf{Distributional-SGD} }
  \text{\bf Input:} $\{x_i\}_{i=1}^N$\\[-4mm]
  \begin{algorithmic}[1]\label{alg:distributional_sgd}
    \STATE Initialize $\Theta_0 = \{W, U, \beta, \rho\}$ randomly.
    \FOR{$i=1,\ldots, t$}
      \STATE Sample $x_i$ uniformly from $\{x_i\}_{i=1}^N$.
      \STATE Sample $\xi_i \sim \Ucal([0, 1]^l)$.      
      \STATE Compute stochastic gradients $\widehat\nabla_{\Theta} \Htil(\Theta_{i};x_i)$ or $\widehat{\tilde\nabla}_{\Theta} \Htil(\Theta_{i};x_i)$, defined in~\eq{eq:unbiased_full_grad} and~\eq{eq:biased_full_grad}, respectively.
      \STATE Update parameters as
       $$
       \hspace{-14mm}
       \Theta_{i+1}=\Theta_{i} - \gamma_i \widehat\nabla_{\Theta} \Htil(\Theta_{i};x_i), \text{or}
       $$
       $$
       \Theta_{i+1}=\Theta_{i} - \gamma_i \widehat{\tilde\nabla}_{\Theta} \Htil(\Theta_{i};x_i), \text{respectively}.
       $$
    \ENDFOR\\
   \end{algorithmic}
\end{algorithm}
%
\begin{definition}[Distributional derivative]\cite{Grubb08}\label{def:dist_grad}
Let $u\in \Dcal'(\Omega)$, then a distribution $v$ is called the distributional derivative of $u$, denoted as $v = Du$, if it satisfies
\begin{eqnarray*}
\int_\Omega v\phi dx = -\int_\Omega u \partial \phi dx,\quad \forall \phi\in \Ccal^\infty_0(\Omega).
\end{eqnarray*} 
\end{definition}
It is straightforward to verify that for  given $\xi$, the function $\htil(z,\xi)\in \Dcal'(\Omega)$ and moreover, $D_z \htil(z, \xi) = \delta_\xi(z)$, which is exactly the Dirac-$\delta$ function. Based on the definition of distributional derivatives and chain rules, we are able to compute the distributional derivative of the function $\Htil(\Theta;x)$, which is provided in the following lemma. 
\begin{lemma}\label{lemma:dist_derivative}
For a given sample $x$, the distributional derivative of function $\Htil(\Theta;x)$ w.r.t. $W$ is given by 
\begin{eqnarray}\label{eq:new_grad_I}
D_{W}\Htil(\Theta;x) = \EE_{\xi}\sbr{\Delta_{\htil} \ell(\htil(\sigma(W^\top x), \xi))\sigma(W^\top x)\bullet(1-\sigma(W^\top x)) x^\top},
\end{eqnarray}
where $\bullet$ denotes point-wise product and $\Delta_{\htil} \ell(\htil)$ denotes the finite difference defined as $\sbr{\Delta_{\htil}\ell(\htil)}_{k} = \ell(\htil^{1}_{k}) - \ell(\htil^{0}_{k})$, where $[\htil^{i}_k]_l = \htil_l$ if $k\neq l$, otherwise $[\htil^{i}_k]_l = i$, $i\in\cbr{0, 1}$.
\end{lemma}
We can therefore combine distributional derivative estimators~\eq{eq:new_grad_I} with stochastic gradient descent algorithm (see e.g., ~\cite{NemJudLanSha09} and its variants~\cite{KinBa14,BotCurNoc16}), which we designate as \emph{Distributional SGD}. The detail is presented in Algorithm~\ref{alg:distributional_sgd}, where we denote 
\begin{equation}\label{eq:unbiased_full_grad}
\widehat\nabla_{\Theta} \Htil(\Theta_{i};x_i) = \sbr{\widehat D_W \Htil(\Theta_{i};x_i), \widehat\nabla_{U, \beta, \rho} \Htil(\Theta_{i};x_i)}
\end{equation} 
as the unbiased stochastic estimator of the gradient at $\Theta_{i}$ constructed by sample $x_i, \xi_i$. Compared to the existing algorithms for learning to hash which require substantial effort on optimizing over binary variables, the proposed distributional SGD is much simpler and also amenable to online settings~\cite{HuaYanZhe13,LenWuCheBaietal15}.

In general, the distributional derivative estimator~\eq{eq:new_grad_I} requires two forward passes of the model for \emph{each dimension}. To further accelerate the computation, we approximate the distributional derivative $D_{W}\Htil(\Theta;x)$ by exploiting the mean value theorem and Taylor expansion by 
\begin{eqnarray}\label{eq:new_grad_II}
\Dtil_{W}\Htil(\Theta;x) := \EE_{\xi}\sbr{\nabla_{\htil} \ell(\htil(\sigma(W^\top x), \xi))\sigma(W^\top x)\bullet(1-\sigma(W^\top x)) x^\top}, 
\end{eqnarray}
which can be computed for each dimension in \emph{one} pass. Then, we can exploit this estimator
\begin{equation}\label{eq:biased_full_grad}
\widehat{\tilde\nabla}_\Theta \Htil(\Theta_{i}; x_i) = \sbr{\widehat \Dtil_W \Htil(\Theta_{i};x_i), \widehat\nabla_{U, \beta, \rho} \Htil(\Theta_{i};x_i)}
\end{equation}
in Algorithm~\ref{alg:distributional_sgd}. Interestingly, the approximate stochastic gradient estimator of the stochastic neuron we established through the distributional derivative coincides with the heuristic ``pseudo-gradient'' constructed~\cite{RaiBerAlaDin14}. Please refer to the supplementary material~\ref{appendix:dist_derivative} for details for the derivation of the approximate gradient estimator~\eq{eq:new_grad_II}.

\subsection{Convergence of Distributional SGD}
One caveat here is that due to the potential discrepancy of the distributional derivative and the traditional gradient, whether the distributional derivative is still a descent direction and whether the SGD algorithm integrated with distributional derivative converges or not remains unclear in general. However, for our learning to hash problem, one can easily show that the distributional derivative in~\eq{eq:new_grad_I} is indeed the true gradient. 
\begin{proposition}\label{thm:derivative_relationship}
The distributional derivative $D_W \Htil(\Theta;x)$ is equivalent to the traditional gradient $\nabla_W H(\Theta;x)$.
\end{proposition}
\begin{proof} First of all, by definition, we have $\Htil(\Theta;x)=H(\Theta;x)$. One can easily verify that under mild condition, both $D_W \Htil(\Theta;x)$ and $\nabla_W H(\Theta;x)$ are continuous and $1$-norm bounded. Hence, it suffices to show that for any distribution $u\in \Ccal^1(\Omega)$ and $Du, \nabla u\in \Lcal_1(\Omega)$, $Du=\nabla u$. For any $\phi\in \Ccal_0^\infty(\Omega)$, by definition of the distributional derivative, we have
$\int_\Omega Du\phi dx = -\int_\Omega u \partial \phi dx$. On the other hand, we always have 
$\int_\Omega \nabla u\phi dx = -\int u \partial \phi dx$. 
Hence, $\int_\Omega (Du-\nabla u)\phi dx=0$ for all $\phi\in \Ccal_0^\infty(\Omega)$. By the Du Bois-Reymond's lemma (see Lemma 3.2 in \cite{Grubb08}), we have $Du=\nabla u$. 
\end{proof}
Consequently, the distributional SGD algorithm enjoys the same convergence property as the traditional SGD algorithm. Applying theorem 2.1 in~\cite{GhaLan13}, we arrive at 
\begin{theorem}\label{thm:convergence}
Under the assumption that $H$ is $L$-Lipschitz smooth and the variance of the stochastic distributional gradient~\eq{eq:unbiased_full_grad} is bounded by $\sigma^2$ in the distributional SGD, for the solution $\Theta_R$ sampled from the trajectory $\cbr{\Theta_i}_{i=1}^t$ with probability $P(R=i) =\frac{2\gamma_i - L\gamma_i^2}{\sum_{i=1}^t 2\gamma_i - L\gamma_i^2}$ where $\gamma_i\sim\Ocal\rbr{1/\sqrt{t}}$, we have
{\small
$$
\EE\sbr{\Big\|\nabla_{\Theta} \Htil(\Theta_R)\Big\|^2}\sim \Ocal\rbr{\frac{1}{\sqrt{t}}}.
$$ 
}
\end{theorem}
In fact, with the approximate gradient estimators~\eq{eq:new_grad_II}, the proposed algorithm is also converging in terms of first-order conditions, \ie,
\begin{theorem}\label{thm:biased_convergence}
Under the assumption that the variance of the approximate stochastic distributional gradient~\eq{eq:biased_full_grad} is bounded by $\sigma^2$, for the solution $\Theta_R$ sampled from the trajectory $\cbr{\Theta_i}_{i=1}^t$ with probability $P(R =i)=\frac{\gamma_i}{\sum_{i=1}^t\gamma_i}$ where $\gamma_i\sim\Ocal\rbr{1/\sqrt{t}}$, we have
{\small
$$
\EE\sbr{\rbr{\Theta_R - \Theta^*}^\top \tilde\nabla_\Theta\Htil(\Theta_R)}\sim\Ocal\rbr{\frac{1}{\sqrt{t}}},
$$
}
where $\Theta^*$ denotes the optimal solution. 
\end{theorem}
For the detailed proof of theorem~\ref{thm:convergence} and~\ref{thm:biased_convergence}, please refer to the supplementary material~\ref{appendix:convergence_dist_sgd}.

\section{Connections}\label{sec:connection}

The proposed stochastic generative hashing is a general framework. In this section, we reveal the connection to several existing algorithms.

\noindent\textbf{Iterative Quantization~(ITQ).} If we fix some $\rho >0$, and $U = WR$ where $W$ is formed by eigenvectors of the covariance matrix and $R$ is an orthogonal matrix, we have $U^\top U = I$. If we assume the joint distribution as
$$
p(x, h) \propto \Ncal(WRh, \rho^2I)\Bcal(\theta),
$$
and parametrize $q(h|x_i) = \delta_{b_i}(h)$, then from the objective in~\eq{eqn:helmholtz} and ignoring the irrelevant terms, we obtain the optimization
\vspace{-2mm}
\begin{equation}\label{eqn:itq}
\min_{R, b} \sum_{i=1}^{N} \|x_i - WRb_i\|^2,
\end{equation}
which is exactly the objective of iterative quantization~\cite{itq}.

\noindent\textbf{Binary Autoencoder~(BA).} If we use the deterministic linear encoding function, \ie, $q(h|x) = \delta_{\frac{1 + \sgn(W^\top x)}{2}}(h)$, and prefix some $\rho >0$, and ignore the irrelevant terms, the optimization~\eq{eqn:helmholtz} reduces to
\begin{equation}\label{eqn:ba}
\min_{U, W} \sum_{i=1}^N \Big\|x_i - Uh\Big\|^2, \,\, \st \,\,h = \frac{1 + \sgn(W^\top x)}{2},
\end{equation}
which is the objective of a binary autoencoder~\cite{CarRaz15}. 

In BA, the encoding procedure is deterministic, therefore, the entropy term $\EE_{q(h|x)}\sbr{\log q(h|x)} = 0$. In fact, the entropy term, if non-zero, performs like a regularization and helps to avoid wasting bits. Moreover, without the stochasticity, the optimization~\eq{eqn:ba} becomes extremely difficult due to the binary constraints. While for the proposed algorithm, we exploit the stochasticity to bypass such difficulty in optimization. The stochasticity enables us to accelerate the optimization as shown in section~\ref{sec:empirical_sgd}.

\section{Experiments}\label{sec:experiments}

In this section, we evaluate the performance of the proposed distributional SGD on commonly used datasets in hashing. Due to the efficiency consideration, we conduct the experiments mainly with the approximate gradient estimator~\eq{eq:new_grad_II}. We evaluate the model and algorithm from several aspects to demonstrate the power of the proposed SGH: \textbf{(1) Reconstruction loss.} To demonstrate the flexibility of generative modeling, we compare the $L2$ reconstruction error to that of ITQ~\cite{itq}, showing the benefits of modeling without the orthogonality constraints. \textbf{(2) Convergence of the distributional SGD}. We evaluate the reconstruction error showing that the proposed algorithm indeed converges, verifying the theorems. \textbf{(3) Training time.} The existing generative works require a significant amount of time for training the model. In contrast, our SGD algorithm is very fast to train both in terms of number of examples needed and the wall time.  \textbf{(4) Nearest neighbor retrieval.} We show Recall K@N plots on standard large scale nearest neighbor search benchmark datasets of \texttt{MNIST}, \texttt{SIFT-1M}, \texttt{GIST-1M} and \texttt{SIFT-1B}, for all of which we achieve state-of-the-art among binary hashing methods. \textbf{(5) Reconstruction visualization.} Due to the generative nature of our model, we can regenerate the original input with very few bits. On \texttt{MNIST} and \texttt{CIFAR10}, we qualitatively illustrate the templates that correspond to each bit and the resulting reconstruction.

We used several benchmarks datasets, \ie, (1) \texttt{MNIST} which contains 60,000 digit images of size $28 \times 28$ pixels, (2) \texttt{CIFAR-10} which contains 60,000 $32\times32$ pixel color images in 10 classes, (3) \texttt{SIFT-1M} and (4) \texttt{SIFT-1B} which contain $10^6$ and $10^9$ samples, each of which is a $128$ dimensional vector, and (5) \texttt{GIST-1M} which contains $10^6$ samples, each of which is a $960$ dimensional vector. 	

\subsection{Reconstruction loss}\label{sec:reconstruction}

%
\begin{figure}[t]
\begin{center}
  \begin{tabular}{cc}
    \includegraphics[width=0.4\columnwidth, trim={0.25cm 0.1cm 1.15cm 0.6cm},clip]{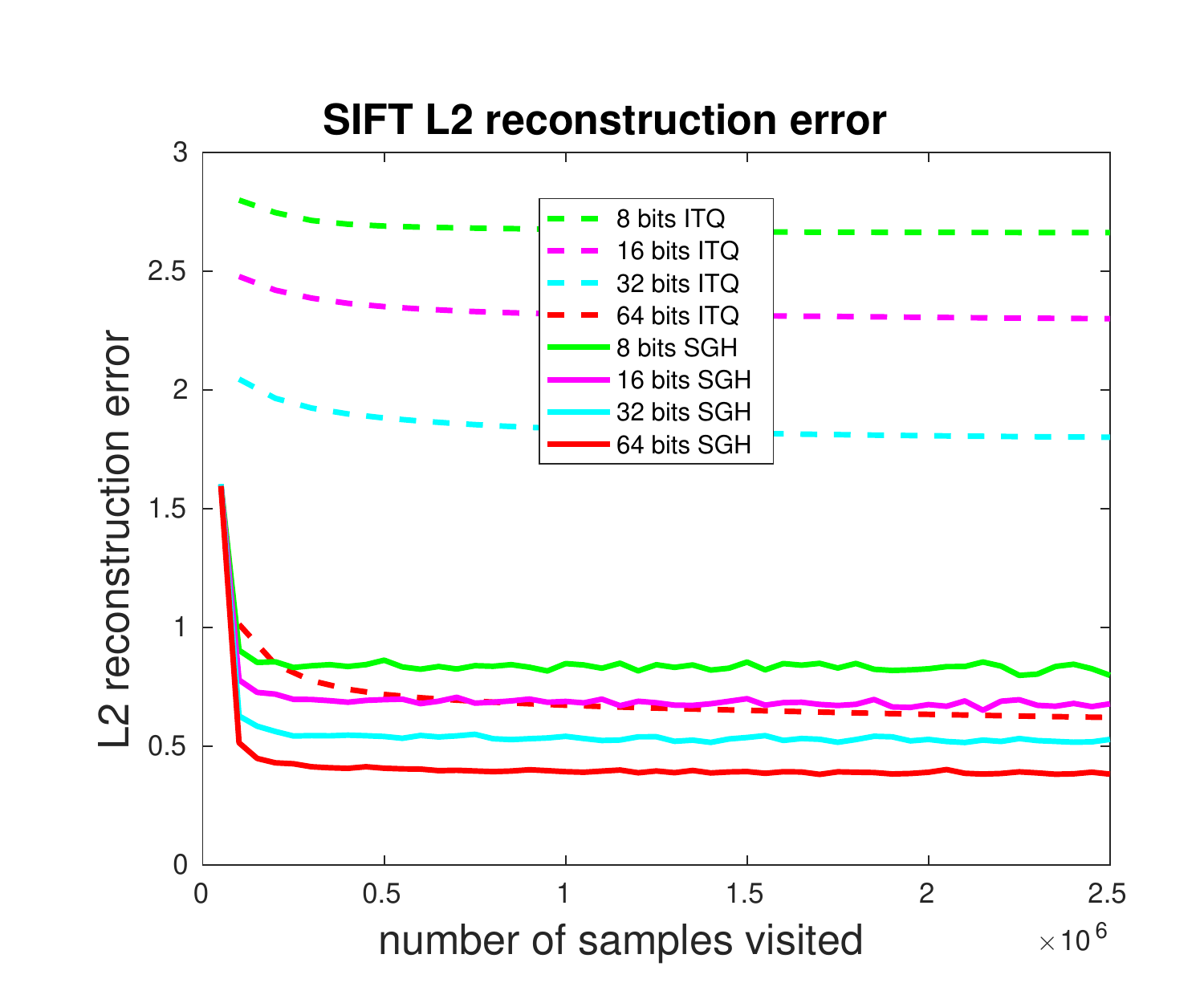}&
    \includegraphics[width=0.4\columnwidth, trim={0.3cm 0.8cm 1.8cm 0.6cm},clip]{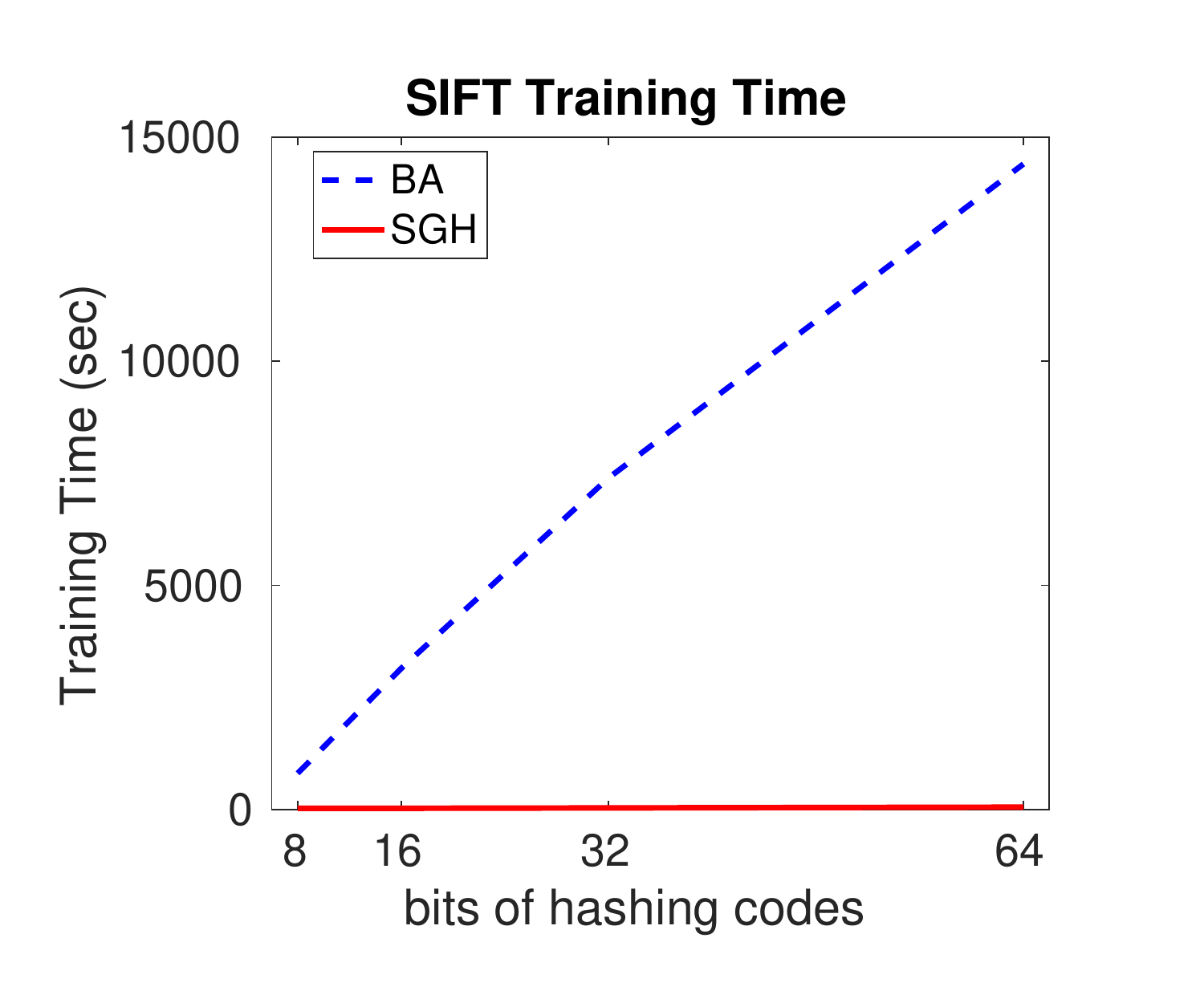}\\
    (a) {Reconstruction Error} &\hspace{-4mm} (b)  {Training Time}  \\
  \end{tabular}
  \vspace{-2mm}
  \caption{(a) Convergence of reconstruction error with number of samples seen by SGD, and (b) training time comparison of BA and SGH on \texttt{SIFT-1M} over the course of training with varying number of bits. }
  \label{fig:reconstruction}
\end{center}
\end{figure}
\begin{table}[t]
\vspace{-1mm}
\caption{Training time on \texttt{SIFT-1M} in second.} \label{table:time_comp}
\vspace{-5mm}
\begin{center}
    \begin{tabular}{| l | l | l | l | l |}
    \hline
    Method & 8 bits & 16 bits & 32 bits  & 64 bits\\ \hline
    SGH & 28.32 &  29.38 & 37.28 & 55.03\\
    ITQ & 92.82 & 121.73 & 173.65 & 259.13\\
    \hline
    \end{tabular}
\end{center}
\vspace{-1mm}
\end{table}
Because our method has a generative model $p(x|h)$, we can easily compute the regenerated input $\tilde{x}=\argmax p(x|h)$, and then compute the $L_2$ loss of the regenerated input and the original $x$, \ie, $\|x-\tilde{x}\|^2_2$. ITQ also trains by minimizing the binary quantization loss, as described in Equation (2) in~\cite{itq}, which is essentially $L_2$ reconstruction loss when the magnitude of the feature vectors is compatible with the radius of the binary cube. We plotted the $L_2$ reconstruction loss of our method and ITQ on \texttt{SIFT-1M} in Figure~\ref{fig:reconstruction}(a) and on \texttt{MNIST} and \texttt{GIST-1M} in Figure~\ref{fig:more_reconstruction}, where the x-axis indicates the number of examples seen by the training algorithm and the y-axis shows the average $L_2$ reconstruction loss. The training time comparison is listed in Table~\ref{table:time_comp}. Our method (SGH) arrives at a better reconstruction loss with comparable or even less time compared to ITQ. The lower reconstruction loss demonstrates our claim that the flexibility of the proposed model afforded by removing the orthogonality constraints indeed brings extra modeling ability. Note that ITQ is generally regarded as a technique with  fast training among the existing binary hashing algorithms, and most other algorithms~\cite{kmeanshashing, sphericalhashing, CarRaz15} take much more time to train.

\subsection{Empirical study of Distributional SGD}\label{sec:empirical_sgd}

We demonstrate the convergence of the distributional derivative with Adam~\cite{KinBa14} numerically on \texttt{SIFT-1M}, \texttt{GIST-1M} and \texttt{MINST} from $8$ bits to $64$ bits. The convergence curves on \texttt{SIFT-1M} are shown in Figure~\ref{fig:reconstruction} (a). The results on \texttt{GIST-1M} and \texttt{MNIST} are similar and shown in Figure~\ref{fig:more_reconstruction} in supplementary material~\ref{appendix:more_exp}. Obviously, the proposed algorithm, even with a biased gradient estimator, converges quickly, no matter how many bits are used. It is reasonable that with more bits, the model fits the data better and the reconstruction error can be reduced further. 

In line with the expectation, our distributional SGD trains much faster since it bypasses integer programming. We benchmark the actual time taken to train our method to convergence and compare that to binary autoencoder hashing~(BA)~\cite{CarRaz15} on \texttt{SIFT-1M}, \texttt{GIST-1M} and \texttt{MINST}. We illustrate the performance on \texttt{SIFT-1M} in Figure~\ref{fig:reconstruction}(b)
. The results on \texttt{GIST-1M} and \texttt{MNIST} datasets follow a similar trend as shown in the supplementary material~\ref{appendix:more_exp}. Empirically, BA takes significantly more time to train on all bit settings due to the expensive cost for solving integer programming subproblem. Our experiments were run on AMD 2.4GHz Opteron CPUs$\times 4$ and 32G memory. Our implementation of the stochastic neuron as well as the whole training procedure was done in TensorFlow. We have released our code on GitHub\footnote{\href{https://github.com/doubling/Stochastic_Generative_Hashing}{https://github.com/doubling/Stochastic\_Generative\_Hashing}}. For the competing methods, we directly used the code released by the authors.

\subsection{Large scale nearest neighbor retrieval}
We compared the stochastic generative hashing on an L2NNS task with several state-of-the-art unsupervised algorithms, including $K$-means hashing~(KMH)~\cite{kmeanshashing}, iterative quantization~(ITQ)~\cite{itq}, spectral hashing~(SH)~\cite{spectralhashing}, spherical hashing~(SpH)~\cite{sphericalhashing}, binary autoencoder~(BA)~\cite{CarRaz15}, and scalable graph hashing~(GH)~\cite{JiaLi15}. We demonstrate the performance of our binary codes by doing standard benchmark experiments of Approximate Nearest Neighbor (ANN) search by comparing the retrieval recall. In particular, we compare with other unsupervised techniques that also generate binary codes. For each query, linear search in \emph{Hamming} space is conducted to find the approximate neighbors. 

Following the experimental setting of~\cite{kmeanshashing}, we plot the Recall10@N curve for \texttt{MNIST}, \texttt{SIFT-1M}, \texttt{GIST-1M}, and \texttt{SIFT-1B} datasets under varying number of bits (16, 32 and 64) in Figure~\ref{fig:recall}. On the \texttt{SIFT-1B} datasets, we only compared with ITQ since the training cost of the other competitors is prohibitive. The recall is defined as the fraction of retrieved true nearest neighbors to the total number of true nearest neighbors. The Recall10@N is the recall of 10 ground truth neighbors in the N retrieved samples.  Note that Recall10@N is generally a more challenging criteria than Recall@N (which is essentially Recall1@N), and better characterizes the retrieval results. For completeness, results of various Recall K@N curves can be found in the supplementary material which show similar trend as the Recall10@N curves.

\begin{figure*}[t]

\begin{center}
  \begin{tabular}{cccc}
    \includegraphics[width=0.245\columnwidth, trim={0.5cm 0.5cm 1cm -0.4cm},clip]{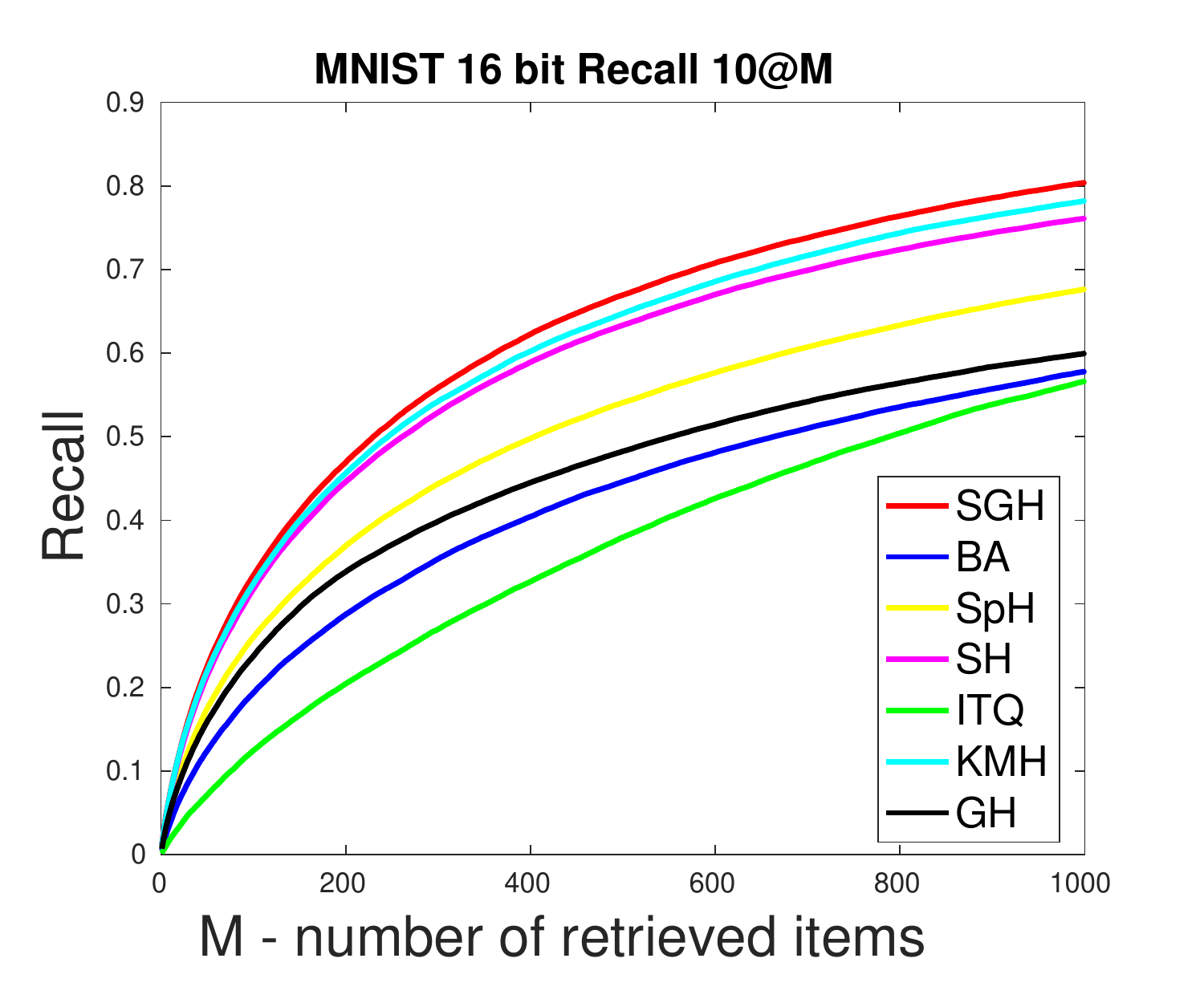}& \hspace{-5mm}
    \includegraphics[width=0.245\columnwidth, trim={0.8cm 1cm 1.6cm 1cm},clip]{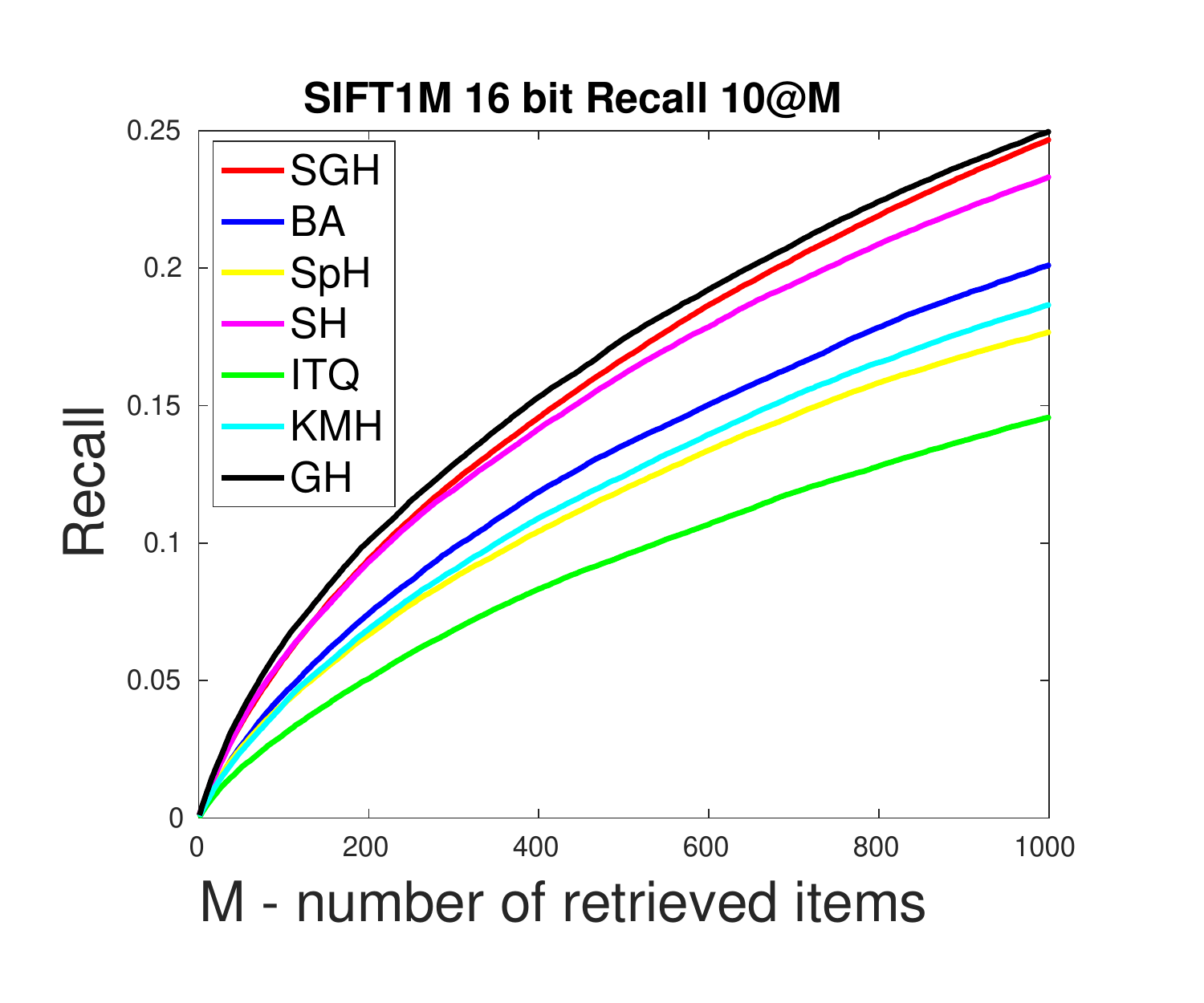} & \hspace{-5mm}
    \includegraphics[width=0.245\columnwidth, trim={0.8cm 1cm 1.6cm 1cm},clip]{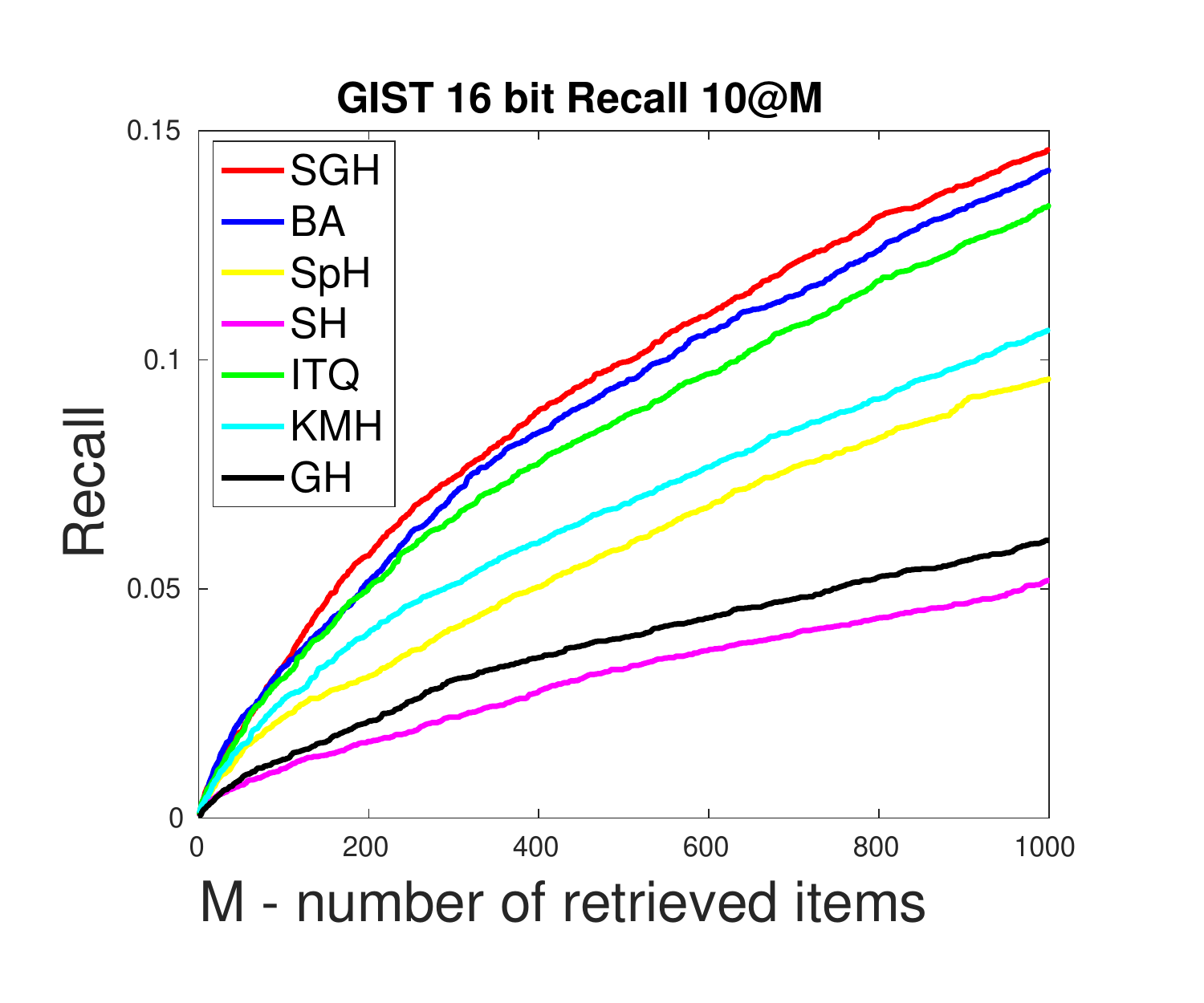} & \hspace{-5mm}
    \includegraphics[width=0.245\columnwidth,  height=0.205\columnwidth, trim={0cm 0cm 0cm 0.1cm},clip]{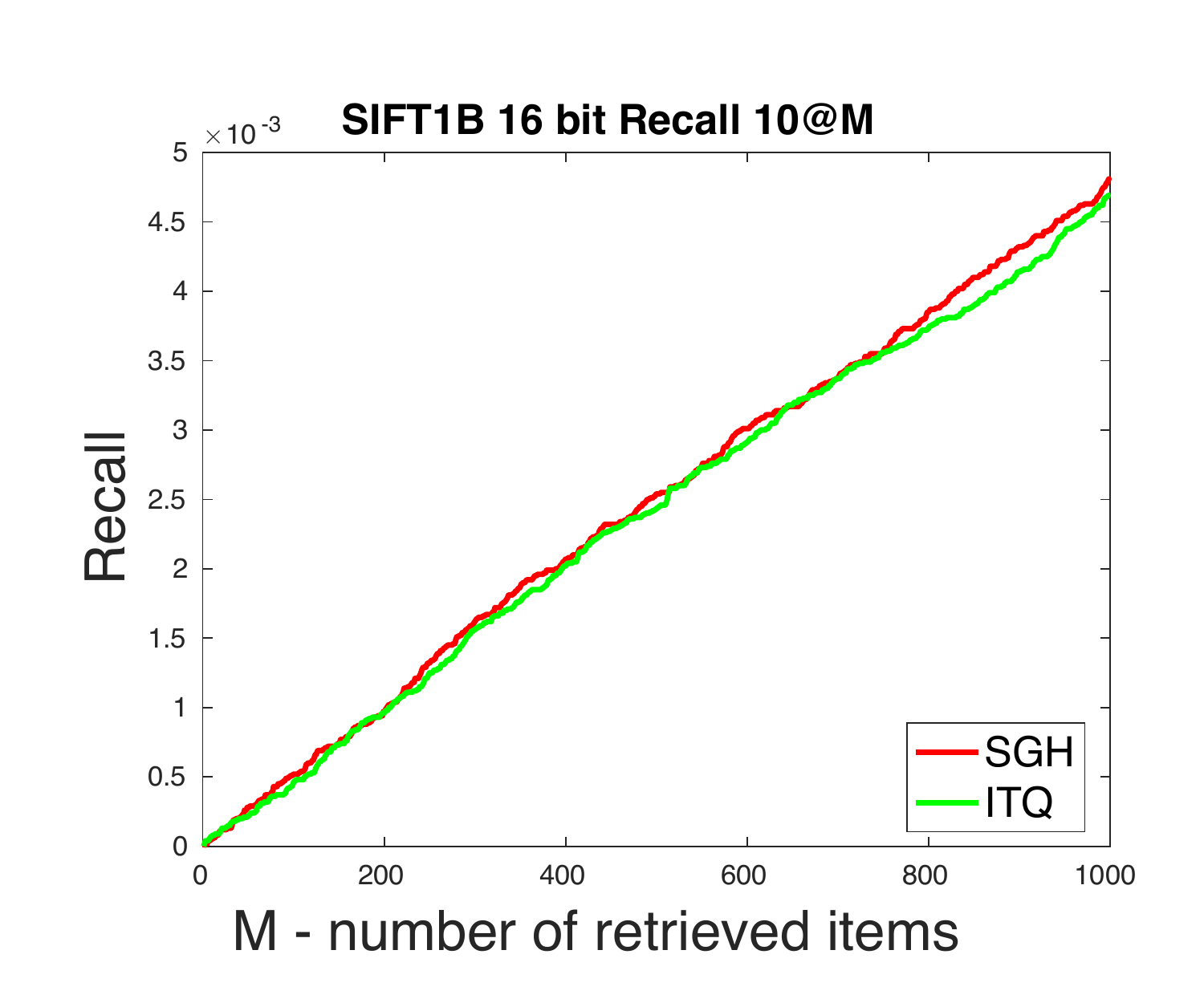} \\
    
    \includegraphics[width=0.245\columnwidth, trim={0.9cm 1cm 1.6cm 1cm},clip]{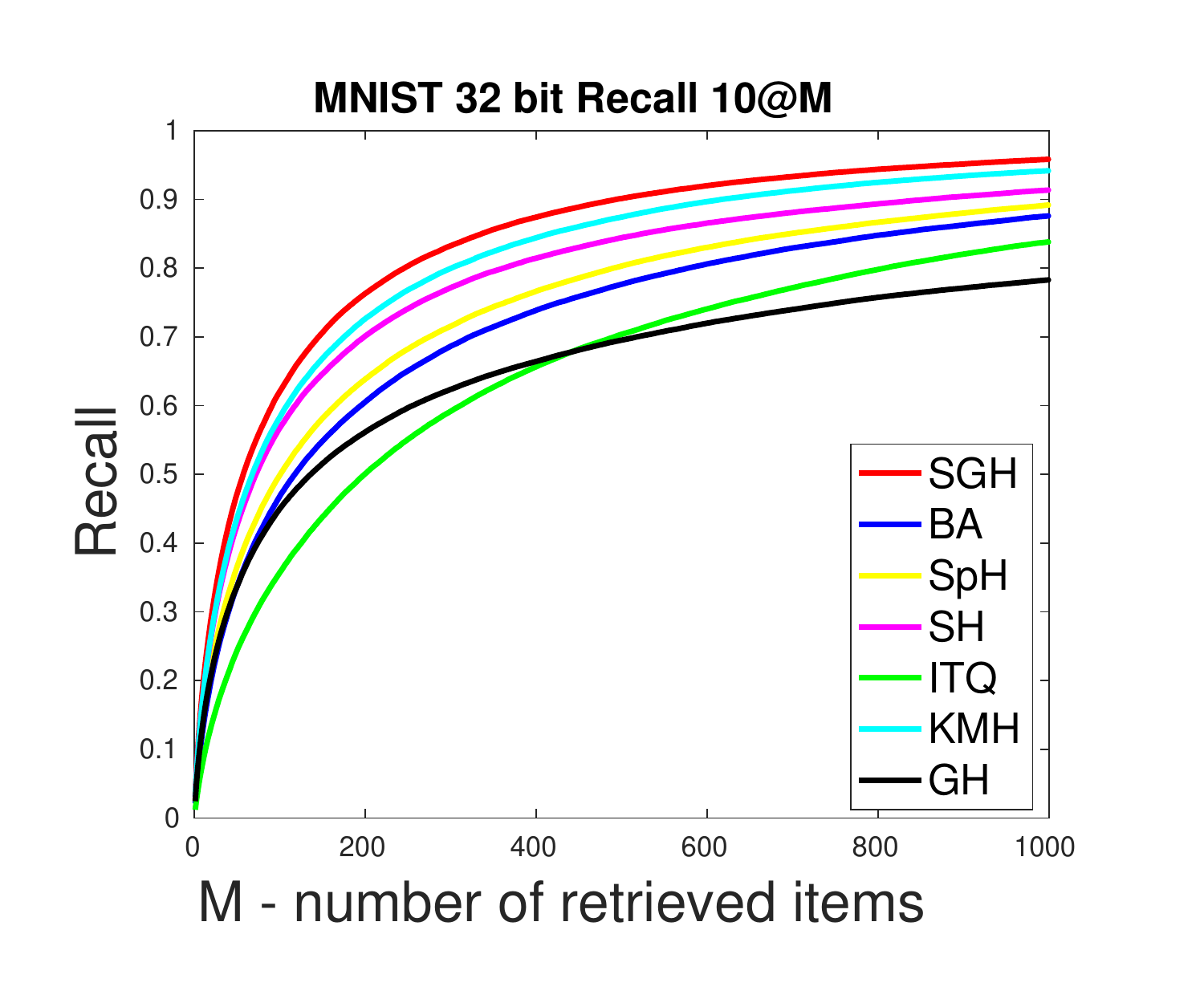} & \hspace{-5mm}
    \includegraphics[width=0.245\columnwidth, trim={0.8cm 1cm 1.6cm 1cm},clip]{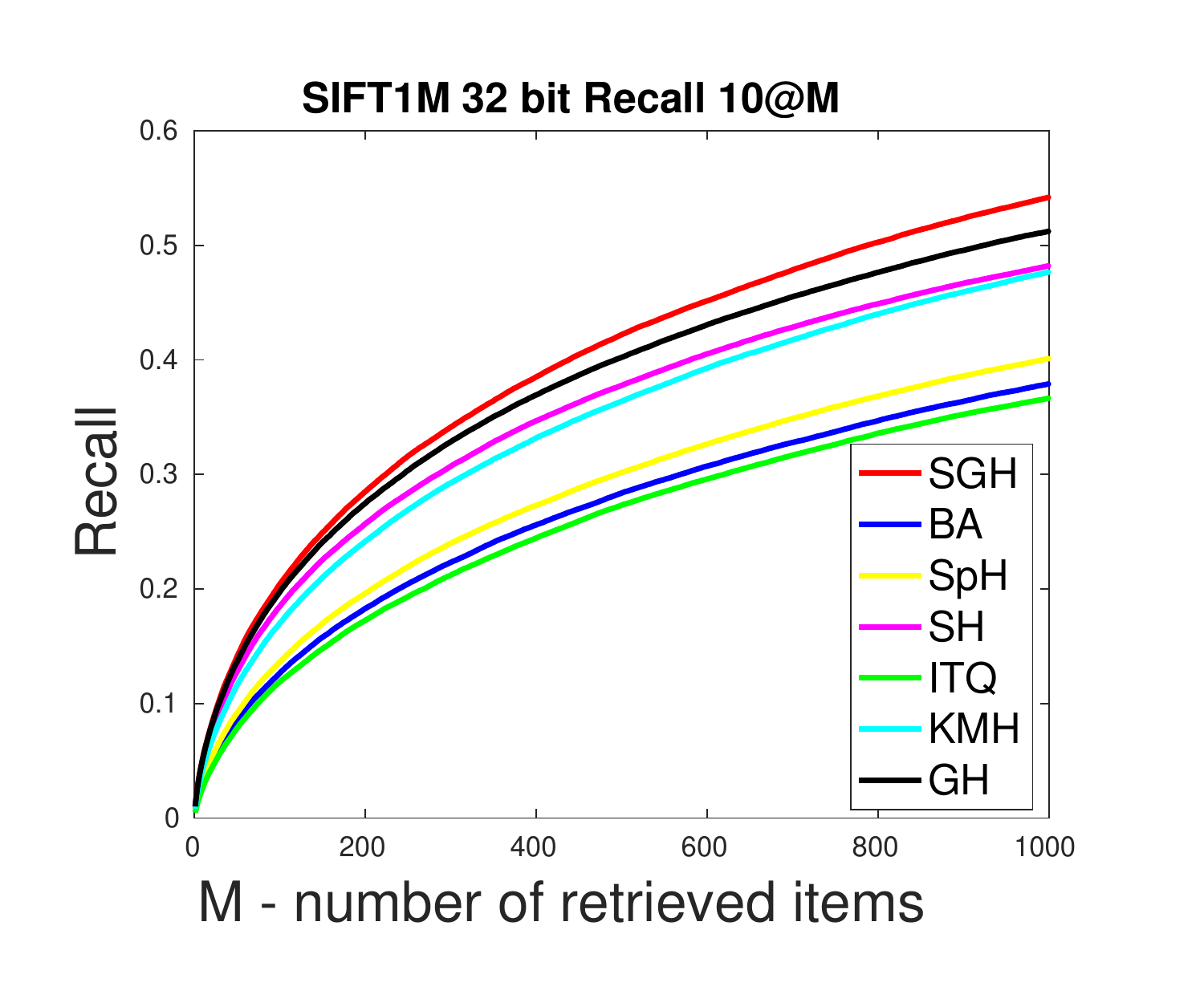} & \hspace{-5mm}
    \includegraphics[width=0.245\columnwidth, trim={0.8cm 1cm 1.6cm 1cm},clip]{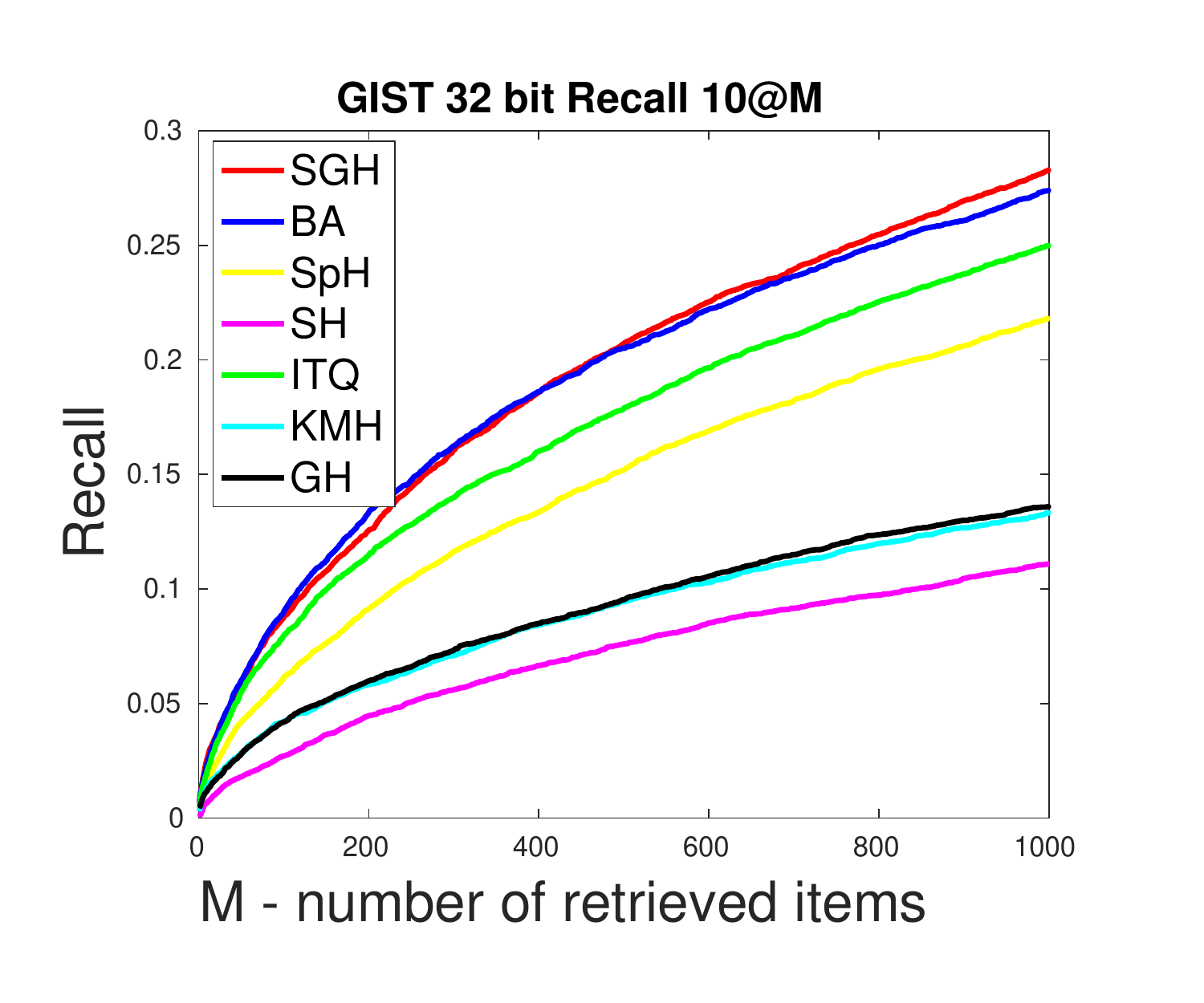} & \hspace{-5mm}
    \includegraphics[width=0.245\columnwidth, height=0.205\columnwidth, trim={0cm -0.1cm 0cm 0.1cm},clip]{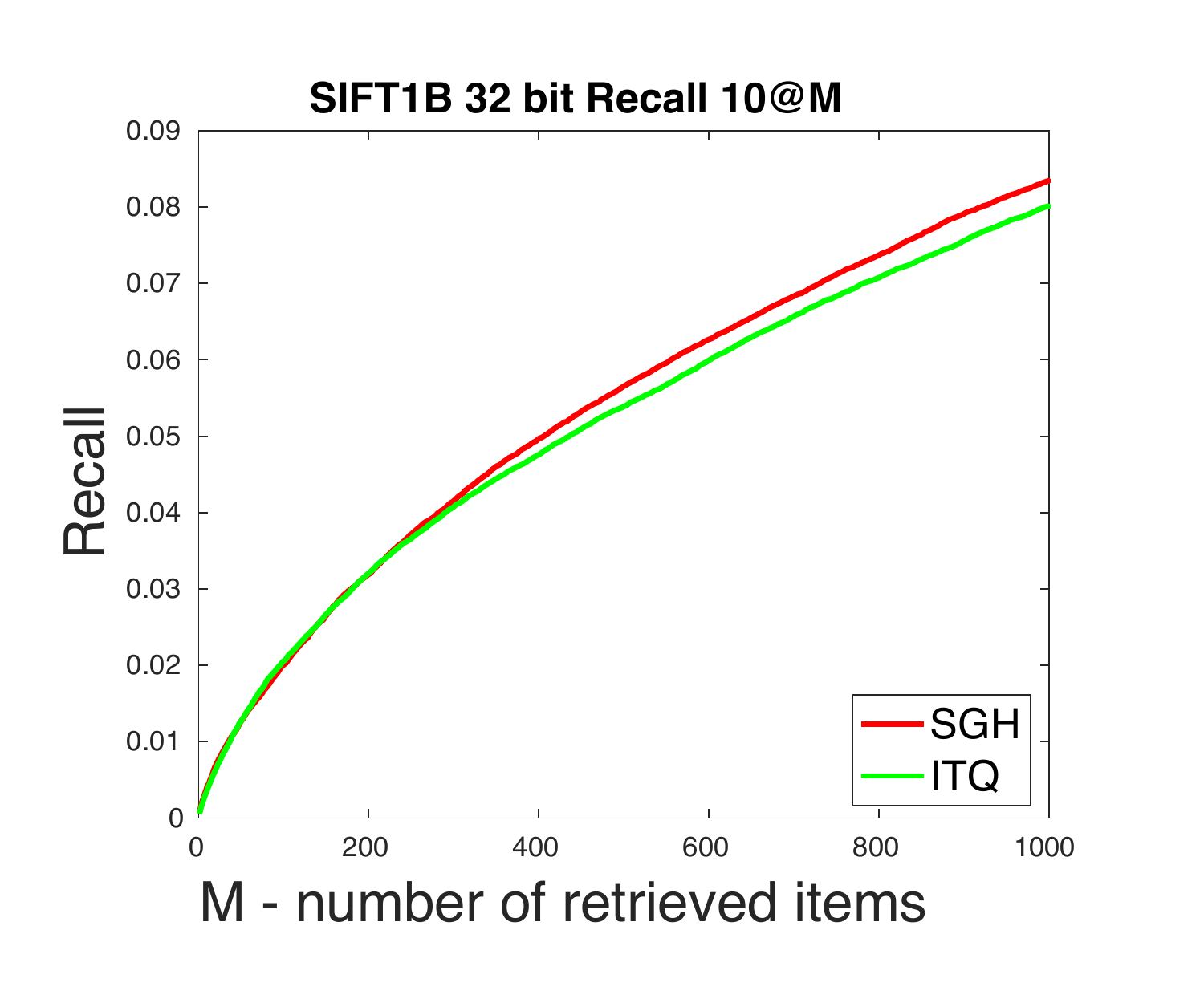} \\
    
    \includegraphics[width=0.245\columnwidth, trim={1cm 0.6cm 1.1cm 1.2cm},clip]{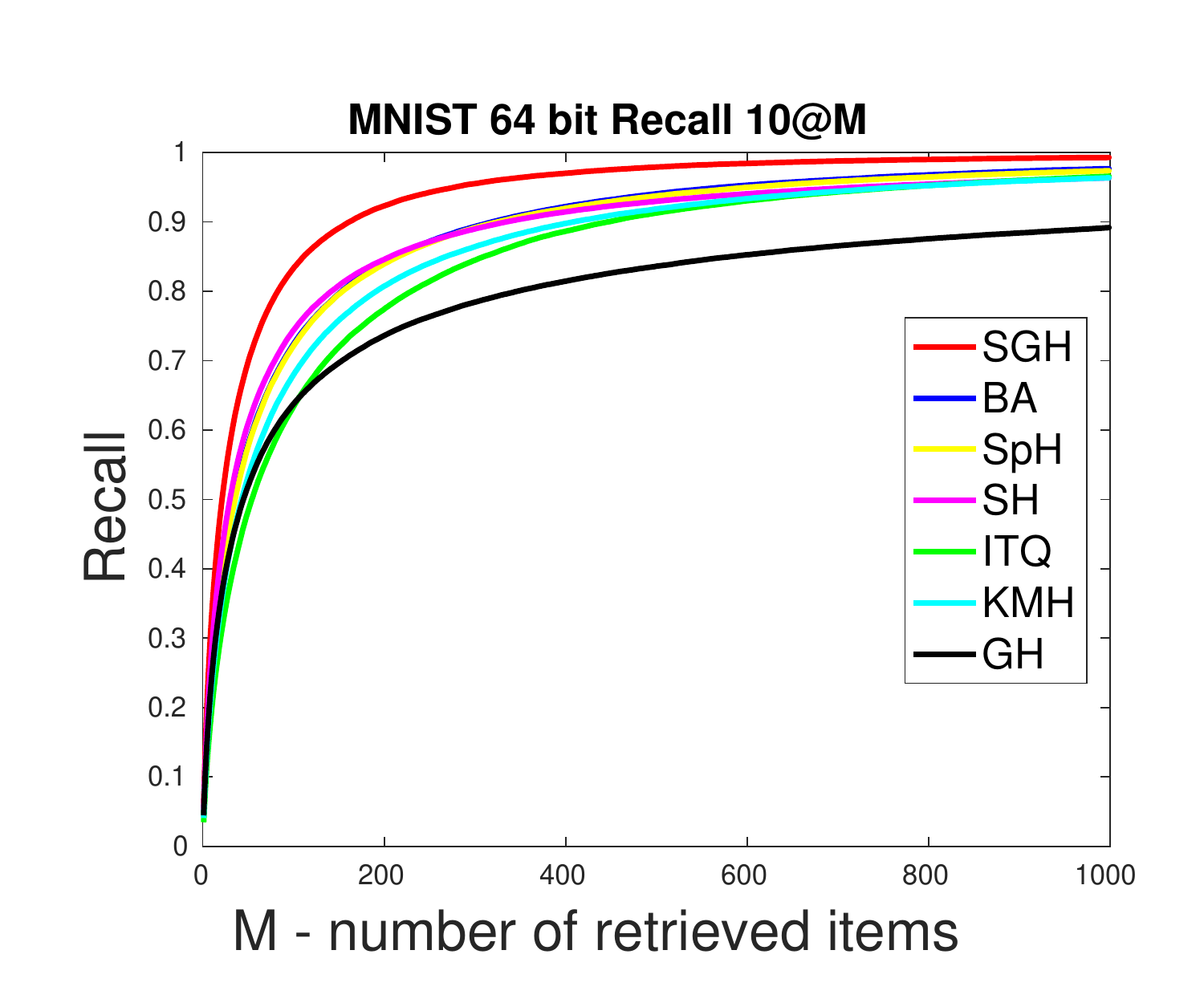} & \hspace{-5mm}
    \includegraphics[width=0.245\columnwidth, trim={0.8cm 1cm 1.6cm 1cm},clip]{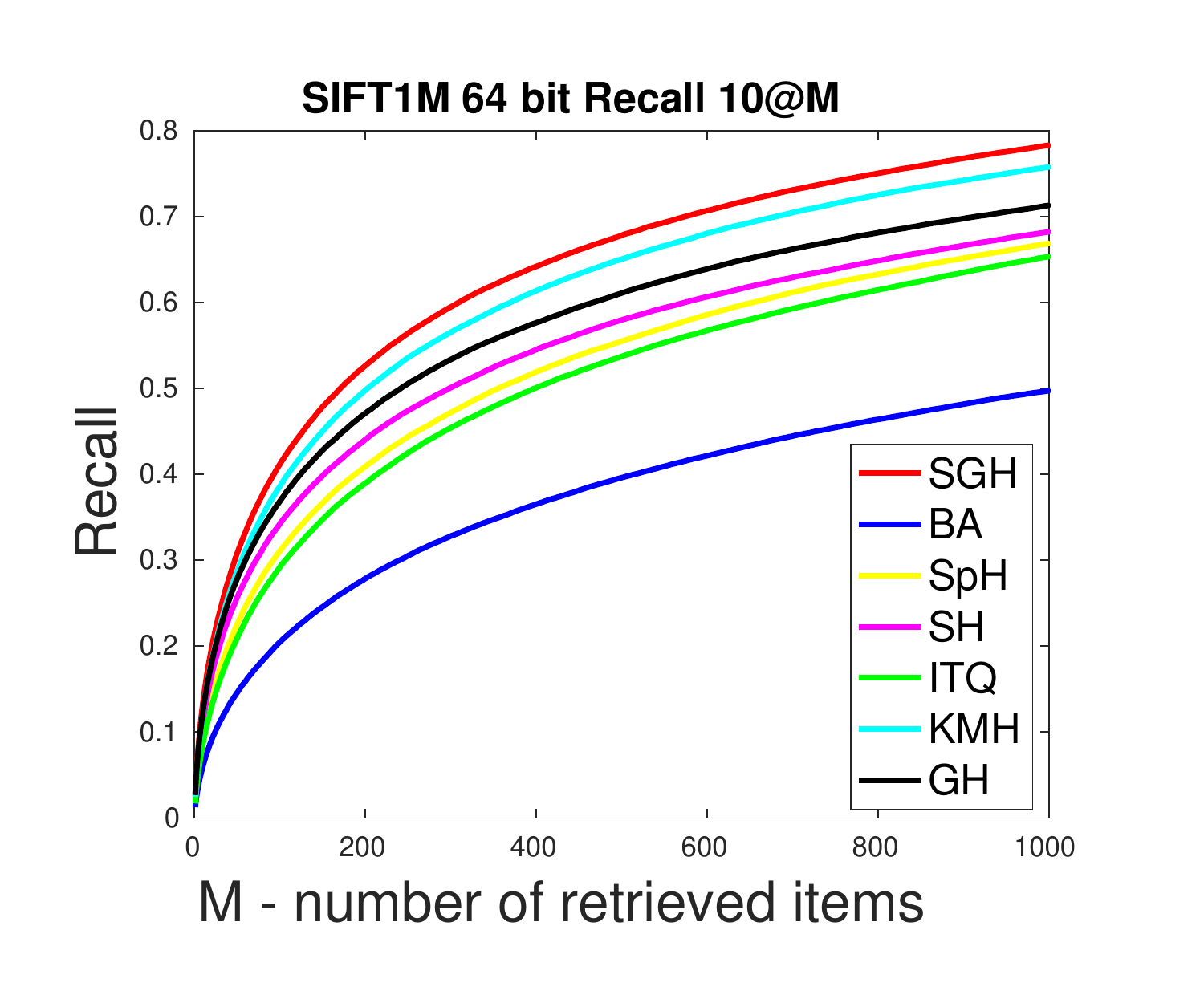} & \hspace{-5mm}
    \includegraphics[width=0.245\columnwidth, trim={0.8cm 1cm 1.6cm 1cm},clip]{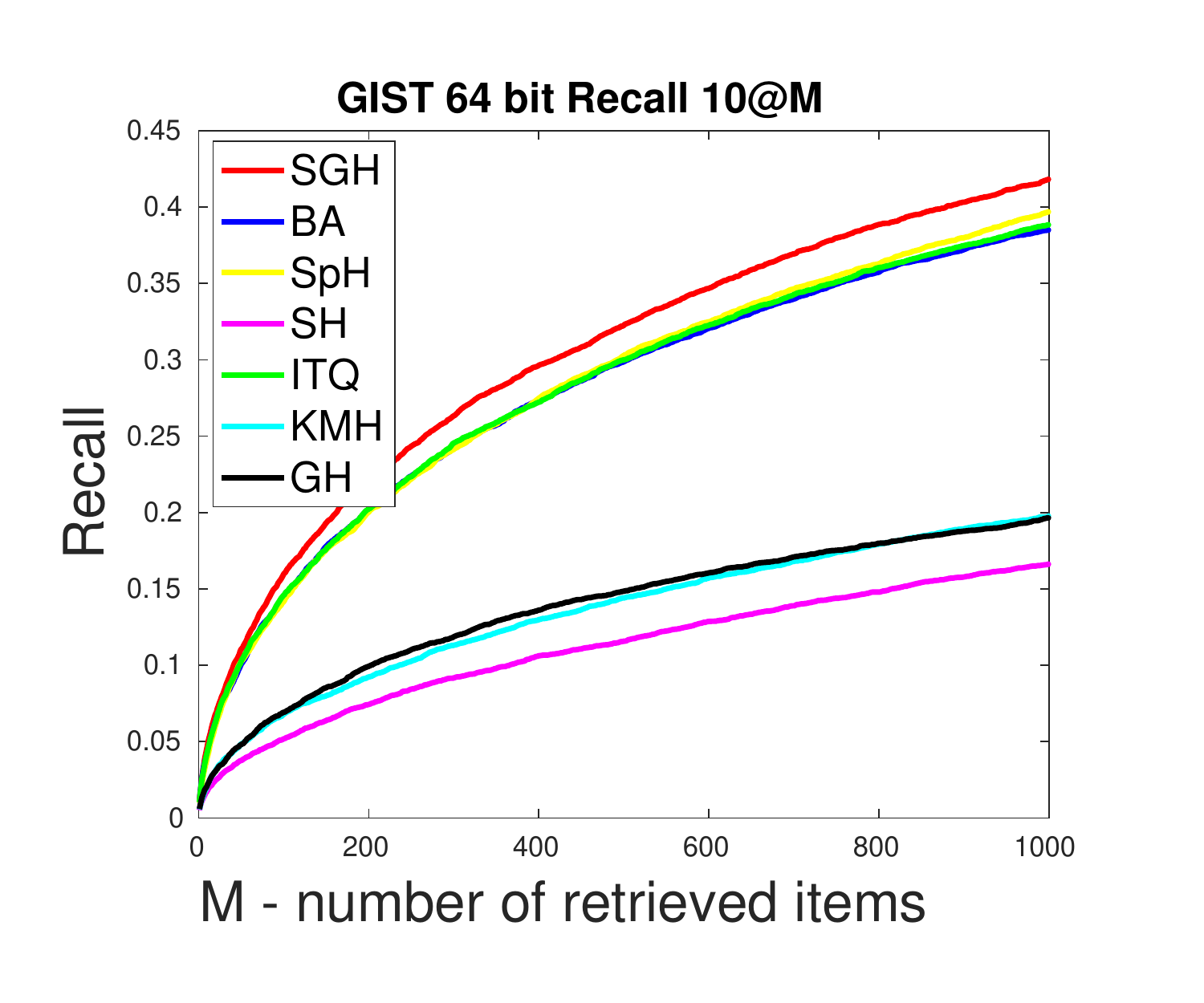} & \hspace{-5mm}
    \includegraphics[width=0.25\columnwidth, height=0.205\columnwidth]{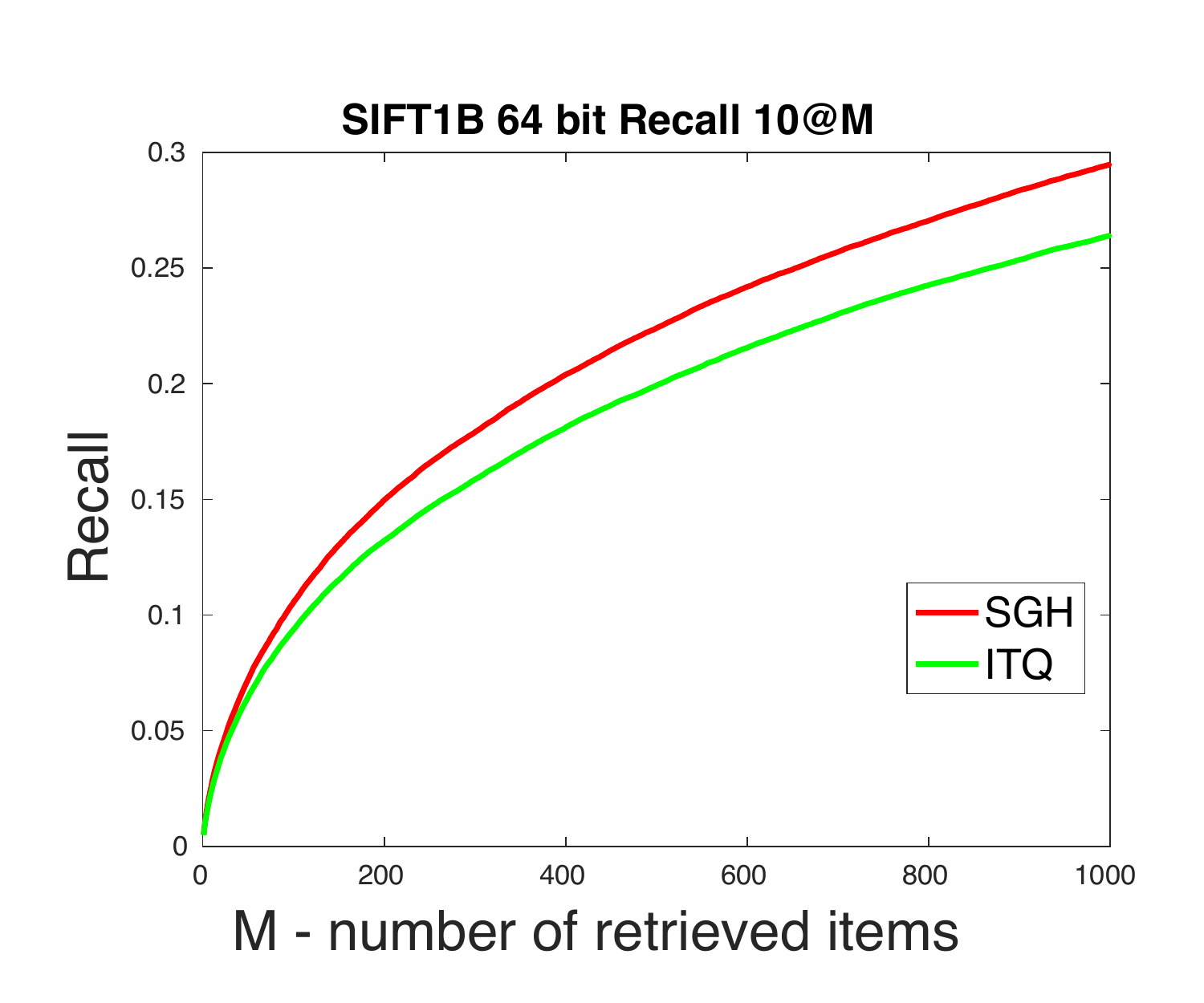} \\
  \end{tabular}
  \vspace{-2mm}
  \caption{L2NNS comparison on \texttt{MNIST}, \texttt{SIFT-1M}, and \texttt{GIST-1M} and \texttt{SIFT-1B} with the length of binary codes varying from $16$ to $64$ bits. We evaluate the performance with Recall 10@$M$ (fraction of top 10 ground truth neighbors in retrieved M), where $M$ increases up to $1000$. }
  \label{fig:recall}
\end{center}
\end{figure*}

Figure~\ref{fig:recall} shows that the proposed SGH consistently performs the best across all bit settings and all datasets. The searching time is the same for the same number of bits, because all algorithms use the same optimized implementation of POPCNT based Hamming distance computation and priority queue. 
We point out that many of the baselines need significant parameter tuning for each experiment to achieve a reasonable recall, except for ITQ and our method, where we fix hyperparameters for all our experiments and used a batch size of $500$ and learning rate of $0.01$ with stepsize decay. Our method is less sensitive to hyperparameters.

\subsection{Visualization of reconstruction}\label{sec:visualization}

\label{sec:visualization}
\begin{figure*}[th]
\centering
\vspace{-1mm}
\subfigure[Templates and re-generated images on \texttt{MNIST}]{
   \includegraphics[width=0.14\textwidth]{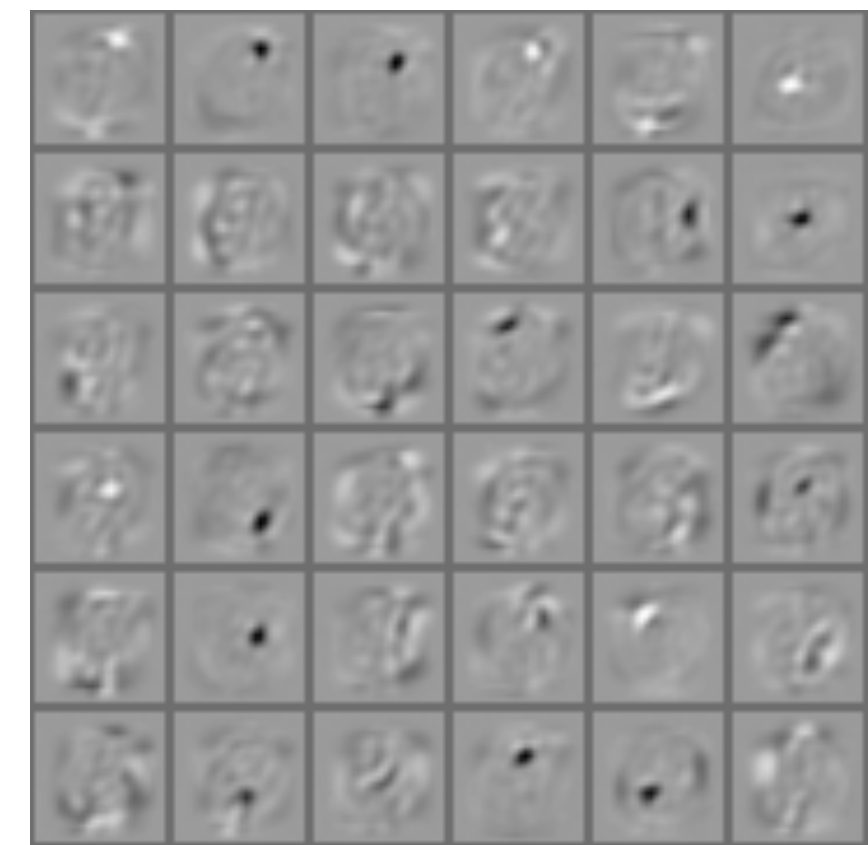}
   \includegraphics[width=0.84\textwidth]{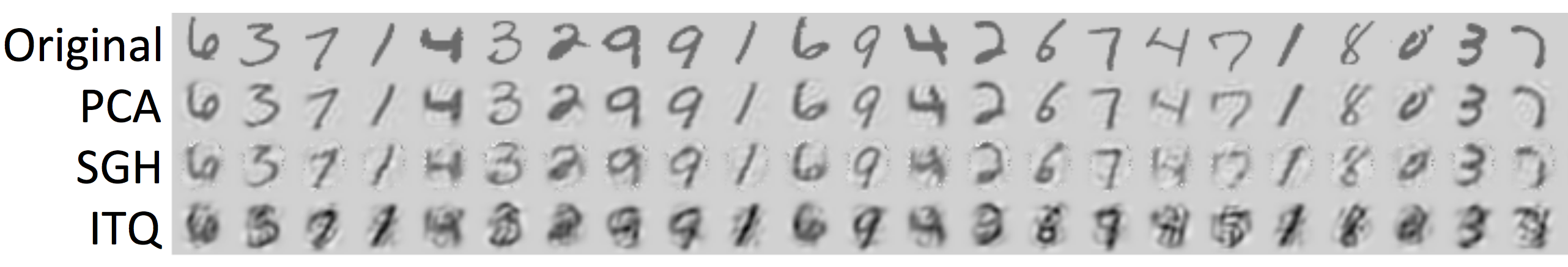}
}
\subfigure[Templates and re-generated images on \texttt{CIFAR-10}]{   
   \includegraphics[width=0.14\textwidth]{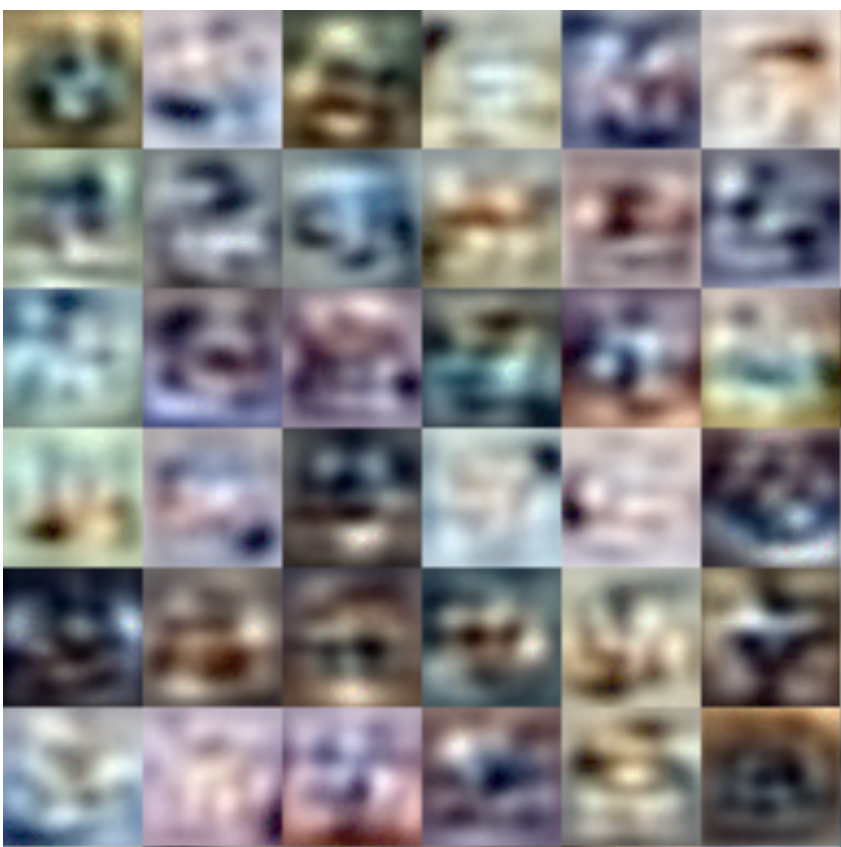}
   \includegraphics[width=0.84\textwidth]{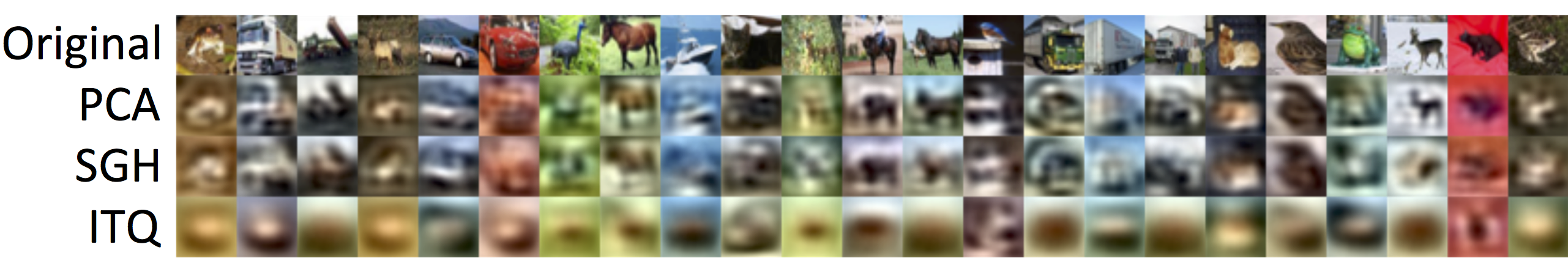}
}
\caption{Illustration of \texttt{MNIST} and \texttt{CIFAR-10} templates (left) and regenerated images (right) from different methods with $64$ hidden binary variables. In \texttt{MNIST}, the four rows and their number of bits used to encode them are, from the top: (1) original image, $28\times 28 \times 8 = 6272$ bits; (2) PCA with 64 components $64\times 32 = 2048$ bits; (3) SGH, 64 bits; (4) ITQ, 64 bits. In \texttt{CIFAR} : (1) original image, $30\times 30 \times 24 = 21600$ bits; (2) PCA with 64 components $64\times 32 = 2048$ bits; (3) SGH, 64 bits; (4) ITQ, 64 bits. The SGH reconstruction tends to be much better than that of ITQ, and is on par with PCA which uses 32 times more bits!}
\label{fig:regenerate}
\end{figure*}

One important aspect of utilizing a generative model for a hash function is that one can generate the input from its hash code. When the inputs are images, this corresponds to image generation, which allows us to visually inspect what the hash bits encode, as well as the differences in the original and generated images. 

In our experiments on \texttt{MNIST} and \texttt{CIFAR-10}, we first visualize the ``template'' which corresponds to each hash bit, \ie, each column of the decoding dictionary $U$. This gives an interesting insight into what each hash bit represents. Unlike PCA components, where the top few look like averaged images and the rest are high frequency noise, each of our image template encodes distinct information and looks much like filter banks of convolution neural networks. Empirically, each template also looks quite different and encodes somewhat meaningful information, indicating that no bits are wasted or duplicated. Note that we obtain this representation as a by-product, without explicitly setting up the model with supervised information, similar to the case in convolution neural nets.

We also compare the reconstruction ability of SGH with the that of ITQ and real valued PCA in Figure~\ref{fig:regenerate}. For ITQ and SGH, we use a $64$-bit hash code. For PCA, we kept 64 components, which amounts to $64\times 32=2048$ bits. Visually comparing with SGH, ITQ reconstructed images look much less recognizable on MNIST and much more blurry on CIFAR-10. Compared to PCA, SGH achieves similar visual quality while using a significantly lower ($32\times$ less) number of bits!

\vspace{-2mm}
\section{Conclusion}\label{sec:conclusion}

In this paper, we have proposed a novel \emph{generative} approach to learn binary hash functions. We have justified from a theoretical angle that the proposed algorithm is able to provide a good hash function that preserves Euclidean neighborhoods, while achieving fast learning and retrieval. Extensive experimental results justify the flexibility of our model, especially in reconstructing the input from the hash codes. Comparisons with approximate nearest neighbor search over several benchmarks demonstrate the advantage of the proposed algorithm empirically. We emphasize that the proposed generative hashing is a general framework which can be extended to semi-supervised settings and other learning to hash scenarios as detailed in the supplementary material. Moreover, the proposed distributional SGD with the unbiased gradient estimator and its approximator can be applied to general integer programming problems, which may be of independent interest.

\section*{Acknowledgements} 
LS is supported in part by NSF IIS-1218749, NIH BIGDATA 1R01GM108341, NSF CAREER IIS-1350983, NSF IIS-1639792 EAGER, ONR N00014-15-1-2340, NVIDIA, Intel and Amazon AWS.


\bibliographystyle{plainnat}
{
}

\clearpage
\newpage

\appendix
\onecolumn

\begin{appendix}

\begin{center}
{\Large \bf Supplementary Material}
\end{center}

\section{Distributional Derivative of Stochastic Neuron}\label{appendix:dist_derivative}

Before we prove the lemma~\ref{lemma:dist_derivative}, we first introduce the chain rule of distributional derivative.
\begin{lemma}\cite{Grubb08}\label{thm:chain_rule} Let $u\in \Dcal'(\Omega)$, we have
\begin{enumerate}
\item {\bf (Chain Rule I)} The distribution derivative of $v = u\circ f$ for any $f(x)\in \Ccal^1: \Omega'\rightarrow \Omega$ is given by $Dv = Du\frac{\partial f}{\partial x}$.
\item {\bf (Chain Rule II)} The distribution derivative of $v = f\circ u$ for any $f(x)\in \Ccal^1(\RR)$ with $f'$ bounded is given by $Dv = f'(u)Du$.
\end{enumerate}
\end{lemma}

\noindent{\bf Proof of Lemma~\ref{lemma:dist_derivative}.} Without loss of generality, we first consider $1$-dimension case. Given $\ell(\htil) : \RR \rightarrow \RR$, $\xi\sim\Ucal(0, 1)$, $\htil : \Omega\rightarrow \cbr{0, 1}$. For $\forall \phi\in \Ccal_{0}^\infty(\Omega)$, we have
\begin{eqnarray*}
&&\int \phi(x) \rbr{D\ell(\htil(x))} dx = -\int \phi'(x) \ell(x)dx\\
&=&-\rbr{\int_{-\infty}^0 \phi'(x) \ell(0)dx + \int_{0}^\infty \phi'(x) \ell(1)dx}\\
&=&-\rbr{\phi(x)\bigg|^{0}_{-\infty} \ell(0) + \phi(x)\bigg|^{\infty}_0 \ell(1)}\\
&=&\rbr{\ell(1) - \ell(0)}\phi(0)
\end{eqnarray*}
where the last equation comes from $\phi\in\Ccal^\infty_0(\Omega)$. We obtain 
$$
D\ell(\htil) = (\ell(1) - \ell(0))\delta(h) := \Delta\ell(h).
$$

We generalize the conclusion to $l$-dimension case with expectation over $\xi$, \ie, $\htil(\cdot, \xi):\Omega \rightarrow \cbr{0, 1}^l$, we have the partial distributional derivative for $k$-th coordinate as 
\begin{eqnarray*}
D_k\EE_{\cbr{\xi_i}_{i=1}^l}\sbr{\ell(\htil(z, \xi))} &=& \EE_{\cbr{\xi_i}_{i=1}^l}\sbr{D_k \ell(\htil(z, \xi))} = \EE_{\cbr{\xi_i}_{i=1, i\neq k}^l}\sbr{(\ell(\htil_k^{1}) - \ell(\htil_k^{0}))}.
\end{eqnarray*}
Therefore, we have the distributional derivative w.r.t. $W$ as 
\begin{eqnarray*}
D\EE_{\cbr{\xi_i}_{i=1}^l}\sbr{\ell(\htil(\sigma(W^\top x), \xi))} &=& \EE_{\cbr{\xi_i}_{i=1}^l}\sbr{D_k \ell(\htil(\sigma(W^\top x), \xi))}\\
\text{chain rule I} 
&=& \EE_{\cbr{\xi_i}_{i=1}^l}\sbr{D_{\htil_k} \ell(\htil(\sigma(W^\top x), \xi))\nabla_W\sigma(W^\top x)}\\
&=& \EE_\xi\sbr{\Delta_{\htil} \ell(\htil(\sigma(W^\top x), \xi))\sigma(W^\top x)\bullet\rbr{1 - \sigma(W^\top x)}x^\top}.
\end{eqnarray*}
\QED

To derive the approximation of the distributional derivative, we exploit the mean value theorem and Taylor expansion. Specifically, for a continuous and differential loss function $\ell(\cdot)$, there exists $\epsilon\in (0, 1)$
\begin{equation*}
\partial_{\htil_k}\ell(\htil)|_{\htil_k = \epsilon} = \sbr{\Delta_{\htil}\ell(\htil) }_k.
\end{equation*}
Moreover, for general smooth functions, we rewrite the $\partial_{\htil_i}\ell(\htil)|_{\htil_i = \epsilon}$ by Taylor expansion, \ie,
\begin{eqnarray*}
\partial_{\htil_k}\ell(\htil)|_{\htil_i = \epsilon} = \partial_{\htil_k}\ell(\htil)|_{\htil_i = 1} + \Ocal(\epsilon)\\
\partial_{\htil_k}\ell(\htil)|_{\htil_i = \epsilon} = \partial_{\htil_k}\ell(\htil)|_{\htil_i = 0} + \Ocal(\epsilon).
\end{eqnarray*} 
we have an approximator as 
\begin{equation}
\partial_{\htil_k}\ell(\htil)|_{\htil_k = \epsilon} \approx \sigma(w_k^\top x)\partial_{\htil_k}\ell(\htil)|_{\htil_k = 1} + (1 - \sigma(w_k^\top x))\partial_{\htil_k}\ell(\htil)|_{\htil_k = 0} = \EE_{\xi}\sbr{\nabla_{\htil}\ell(\htil, \xi)}.
\end{equation}
Plugging into the distributional derivative estimator~\eq{eq:new_grad_I}, we obtain a simple biased gradient estimator,
\begin{equation}
D_{W}\Htil(\Theta;x) \approx \Dtil_{W}\Htil(\Theta;x) := \EE_{\xi}\sbr{\nabla_{\htil} \ell(\htil(\sigma(W^\top x), \xi))\sigma(W^\top x)\bullet(1-\sigma(W^\top x)) x^\top}.
\end{equation}

\section{Convergence of Distributional SGD}\label{appendix:convergence_dist_sgd}

\begin{lemma}\cite{GhaLan13}\label{lemma:general_unbiased_convergence}
Under the assumption that $H$ is $L$-Lipschitz smooth and the variance of the stochastic distributional gradient~\eq{eq:unbiased_full_grad} is bounded by $\sigma^2$, the proposed distributional SGD outputs $\cbr{\Theta_i}_{i=1}^t$,
{\small
$$
\sum_{i=1}^t\rbr{\gamma_i - \frac{L}{2}\gamma_i^2}\EE\sbr{\Big\|\nabla_{\Theta} \Htil(\Theta_i)\Big\|^2}\le \Htil(\Theta_0) - \Htil(\Theta^*) + \frac{L\sigma^2}{2}\sum_{i=1}^t\gamma_i^2,
$$ 
}
where $\Theta_t = \{W_t, U_t, \beta_t, \rho_t\}$. 
\end{lemma}

\noindent{\bf Proof of Theorem~\ref{thm:convergence}.} Lemma~\ref{lemma:general_unbiased_convergence} implies that by randomly sampling a search point $\Theta_R$ with probability $P(R = i) = \frac{2\gamma_i - L\gamma_i^2}{\sum_{i=1}^t 2\gamma_i - L\gamma_i^2}$ where $\gamma_i\sim\Ocal\rbr{1/\sqrt{t}}$ from trajectory $\cbr{\Theta_i}_{i=1}^t$, we have 
$$
\EE\sbr{\nbr{\nabla_\Theta \Htil(\Theta_R)}^2}\sim \Ocal\rbr{\frac{1}{\sqrt{t}}}.
$$
\QED

\begin{lemma}\label{lemma:general_biased_convergence}
Under the assumption that the variance of the approximate stochastic distributional gradient~\eq{eq:biased_full_grad} is bounded by $\sigma^2$, the proposed distributional SGD outputs $\cbr{\Theta_i}_{i=1}^t$ such that
$$
\sum_{i=1}^t\gamma_{i} \EE\sbr{\rbr{\Theta_i - \Theta^*}^\top \tilde\nabla_\Theta\Htil(\Theta_i)} \le \frac{1}{2}\rbr{\EE\sbr{\nbr{\Theta_0 - \Theta^*}^2} + \sum_{i=1}^t\gamma_{i}^2\sigma^2},
$$
where $\Theta^*$ denotes the optimal solution. 
\end{lemma}

\begin{proof}
Denote the optimal solution as $\Theta^*$, we have
\begin{eqnarray*}
\nbr{\Theta_{i+1} -\Theta^*}^2 &=& \nbr{\Theta_{i} - \gamma_{i}\widehat{\tilde\nabla}_\Theta\Htil(\Theta_{i}, x_{i}) - \Theta^*)}^2\\
&=& \nbr{\Theta_{i} - \Theta^*}^2 + \gamma_{i}^2 \nbr{\widehat{\tilde\nabla}_\Theta\Htil(\Theta_{i}, x_{i})}^2 - 2\gamma_{i}\rbr{\Theta_{i} - \Theta^*}^\top \widehat{\tilde\nabla}_\Theta\Htil(\Theta_{i}, x_{i}).
\end{eqnarray*}
Taking expectation on both sides and denoting $a_{j} = \nbr{\Theta_{j} -\Theta^*}^2 $, we have
\begin{eqnarray*}
\EE\sbr{a_{i+1}}\le\EE\sbr{a_{i}}  - 2\gamma_{i}\EE\sbr{\rbr{\Theta_{i} - \Theta^*}^\top \tilde\nabla_\Theta\Htil(\Theta_{i})} + \gamma_{i}^2\sigma^2.
\end{eqnarray*}
Therefore, 
\begin{eqnarray*}
\sum_{i=1}^t\gamma_{i} \EE\sbr{\rbr{\Theta_{i} - \Theta^*}^\top \tilde\nabla_\Theta\Htil(\Theta_{i})} \le \frac{1}{2}\rbr{\EE\sbr{a_0} + \sum_{i=1}^t\gamma_{i}^2\sigma^2}.
\end{eqnarray*}
\end{proof}

\noindent{\bf Proof of Theorem~\ref{thm:biased_convergence}.}
The lemma~\ref{lemma:general_biased_convergence} implies by randomly sampling a search point $\Theta_R$ with probability $P(R=i) = \frac{\gamma_i}{\sum_{i=1}^t\gamma_i}$ where $\gamma_i\sim\Ocal\rbr{1/\sqrt{t}}$ from trajectory $\cbr{\Theta_i}_{i=1}^t$, we have 
$$
\EE\sbr{\rbr{\Theta_R - \Theta^*}^\top \tilde\nabla_\Theta\Htil(\Theta_R)} \le \frac{\EE\sbr{\nbr{\Theta_0 - \Theta^*}^2}+ \sum_{i=1}^t\gamma_{i}^2\sigma^2}{2 \sum_{i=1}^t \gamma_i}\sim\Ocal\rbr{\frac{1}{\sqrt{t}}}.
$$
\QED

\section{More Experiments}\label{appendix:more_exp}

\subsection{Convergence of Distributional SGD and Reconstruction Error Comparison}

\begin{figure*}[th]
\begin{center}
  \begin{tabular}{cc}
    \includegraphics[width=0.325\columnwidth, trim={0.25cm 0cm 0.6cm 1cm},clip]{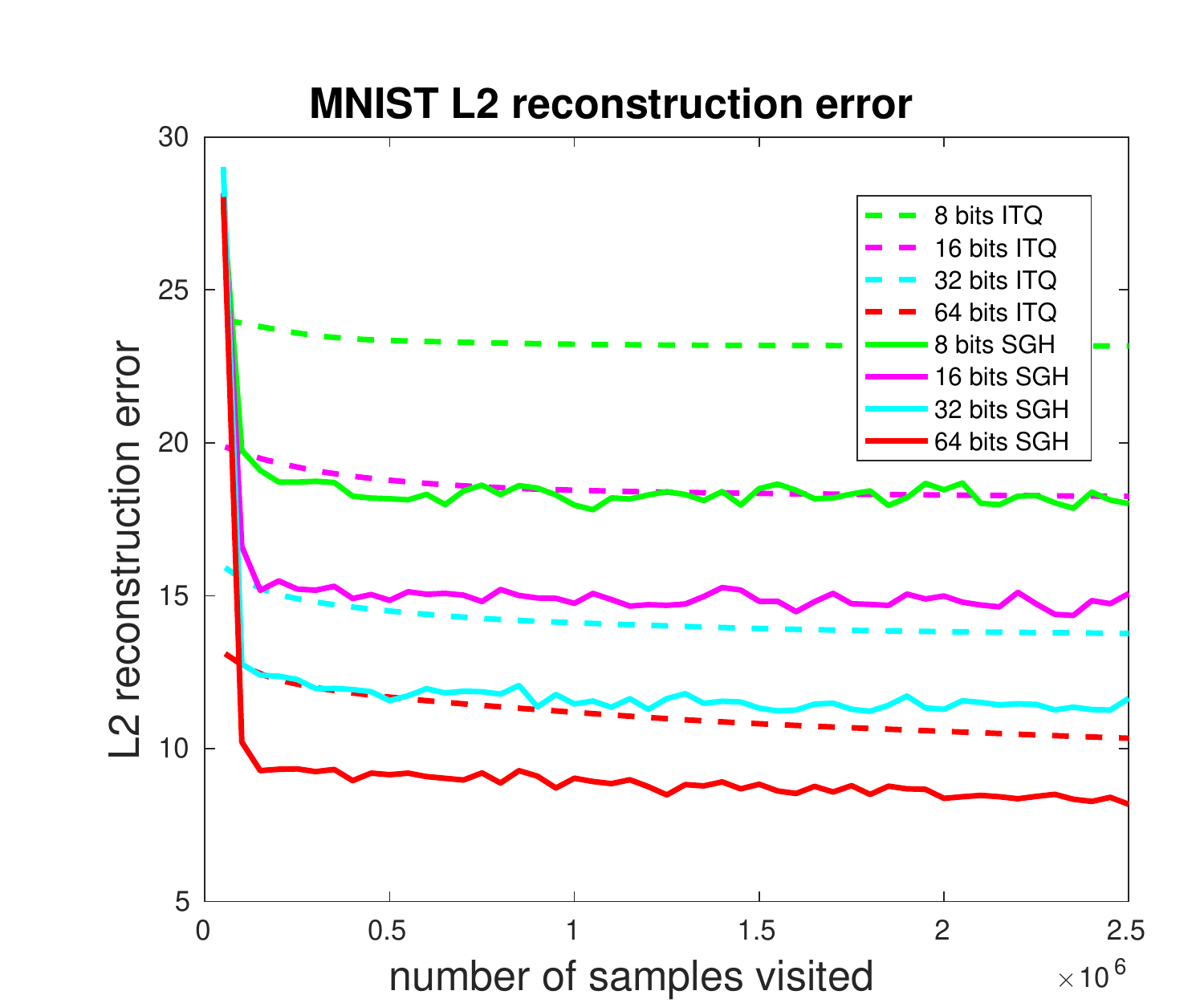}&
 	\hspace{-1mm}
    \includegraphics[width=0.325\columnwidth,  trim={0.25cm 0.8cm 1.5cm 1cm},clip]{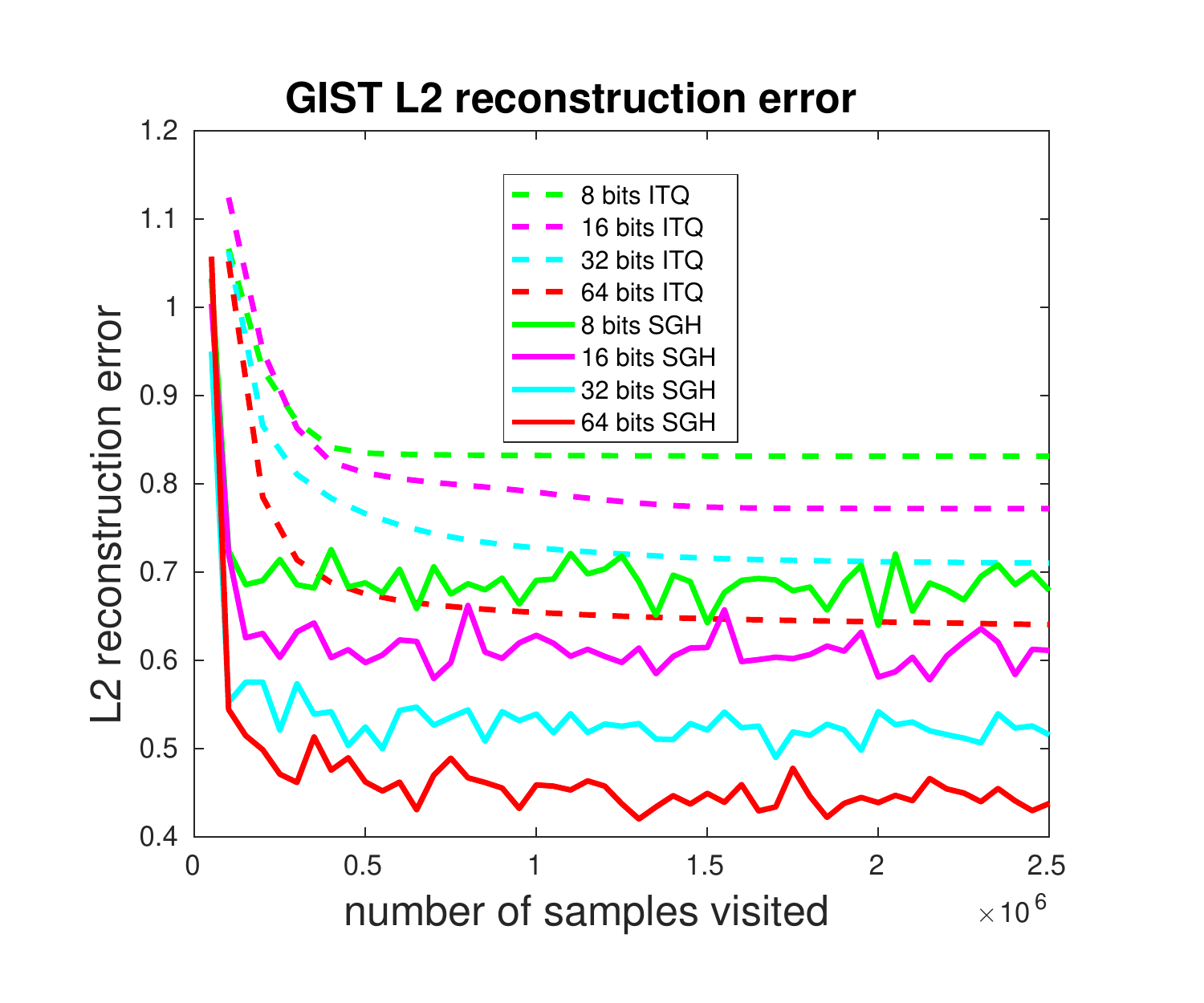}\\
     (a) \texttt{MNIST} &\hspace{-4mm} (b)  \texttt{GIST-1M}  \\
  \end{tabular}
  \caption{L2 reconstruction error convergence on \texttt{MNIST} and \texttt{GIST-1M} of ITQ and SGH over the course of training with varying of the length of the bits (8, 16, 32, 64, respectively). The x-axis represents the number of examples seen by the training algorithm. For ITQ, it sees the training dataset once in one iteration.}
  \label{fig:more_reconstruction}
\end{center}
\end{figure*}

We shows the reconstruction error comparison between ITQ and SGH on \texttt{MNIST} and \texttt{GIST-1M} in Figure~\ref{fig:more_reconstruction}. The results are similar to the performance on \texttt{SIFT-1M}. Because SGH optimizes a more expressive objective than ITQ (without orthogonality) and do not use alternating optimization, it find better solution with lower reconstruction error.

\subsection{Training Time Comparison}

\begin{figure*}[th]
\begin{center}
  \begin{tabular}{cc}
     \includegraphics[width=0.325\columnwidth, trim={0.25cm 0cm 0.6cm 1cm},clip]{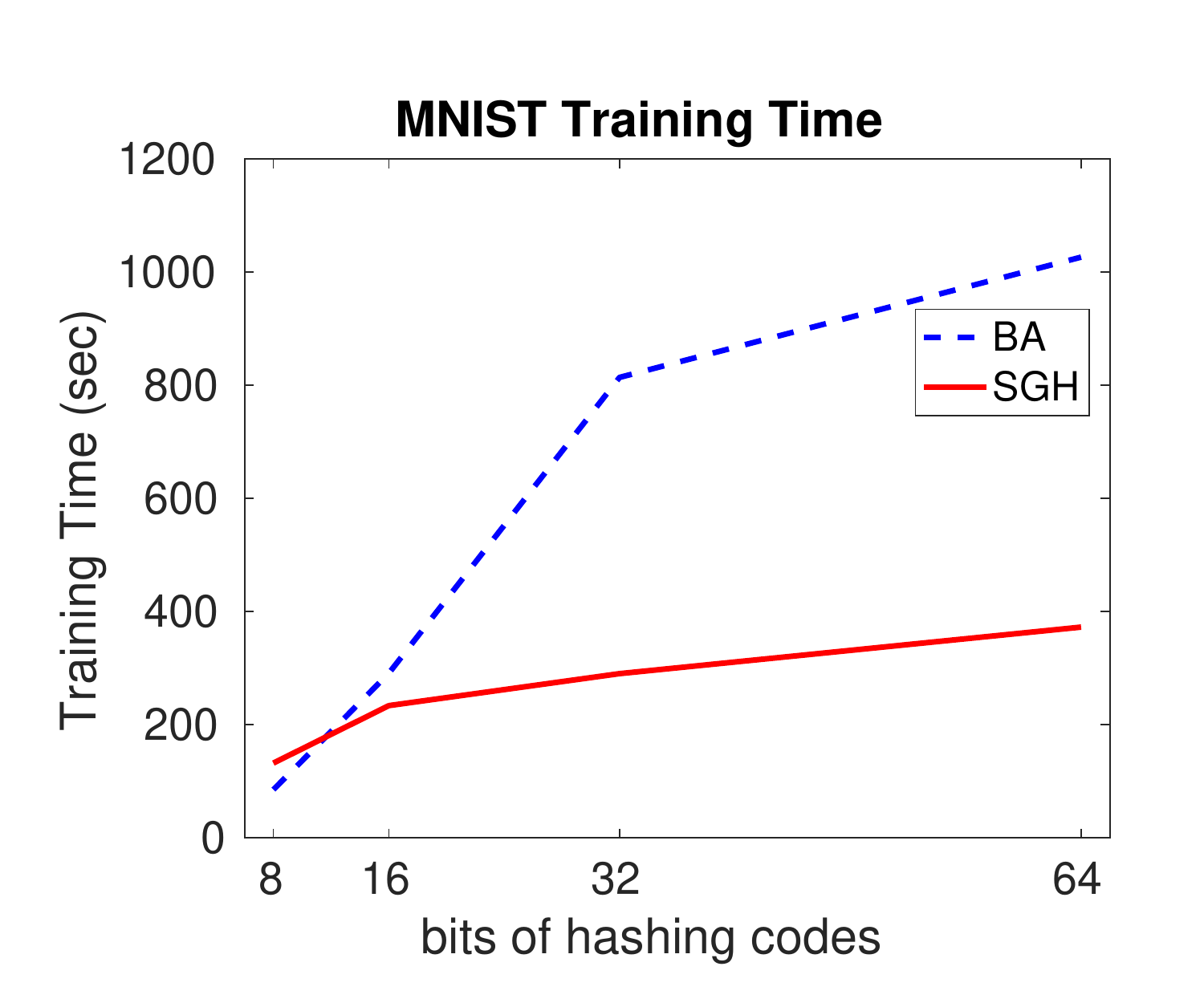}&
    \includegraphics[width=0.325\columnwidth, trim={0.25cm  5.6cm 5.5cm 3cm},clip]{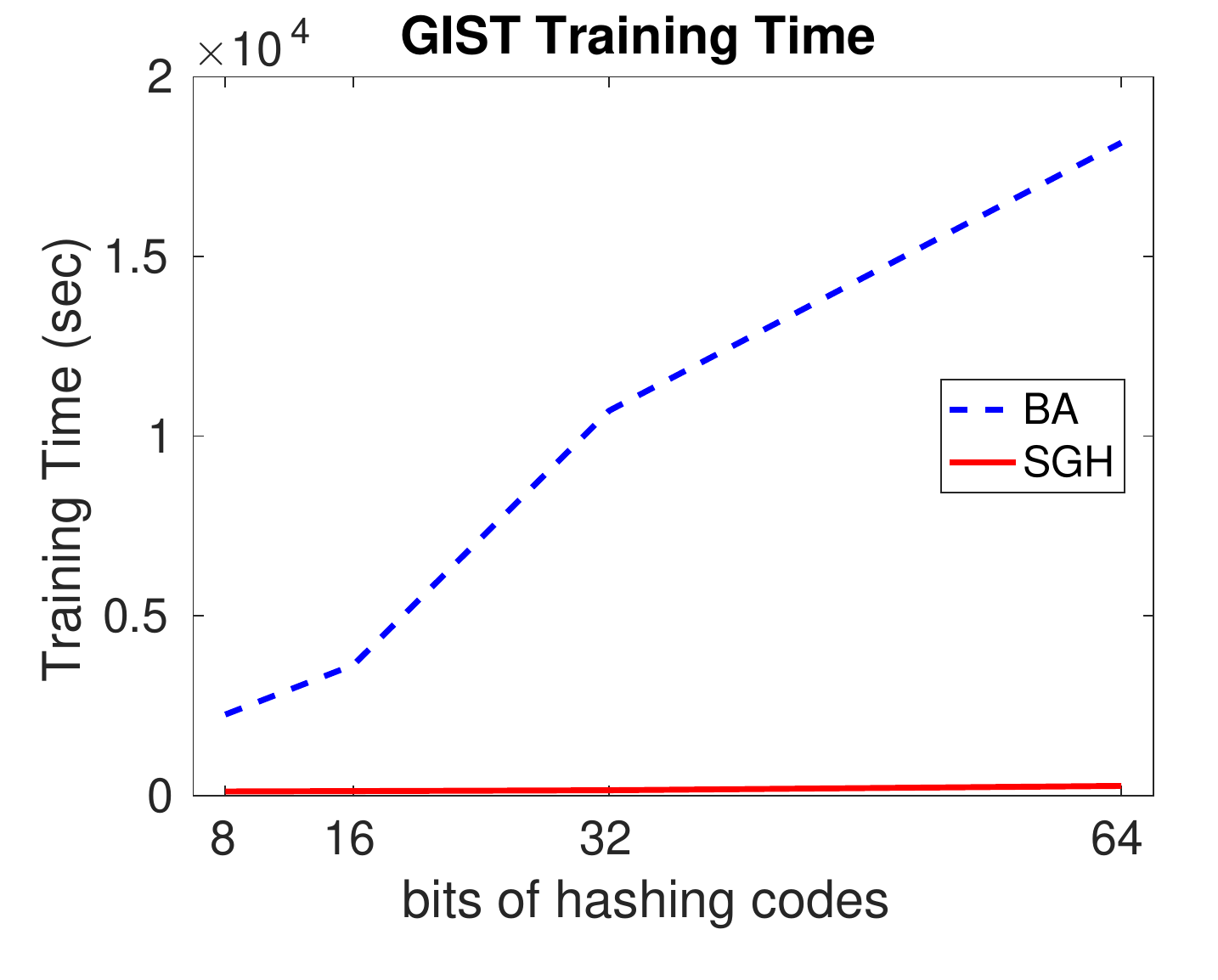}\\
    (a) \texttt{MNIST} &\hspace{-4mm} (b)  \texttt{GIST-1M}  \\
  \end{tabular}
  \caption{Training time comparison between BA and SGH on \texttt{MNIST} and \texttt{GIST-1M}.}
  \label{fig:more_train_time}
\end{center}
\end{figure*}
\begin{figure*}[t]
\begin{center}
\begin{tabular}{ccc}
    \includegraphics[width=0.3\columnwidth, height=0.255\columnwidth, trim={0.25cm 0cm 1cm 0.6cm},clip]{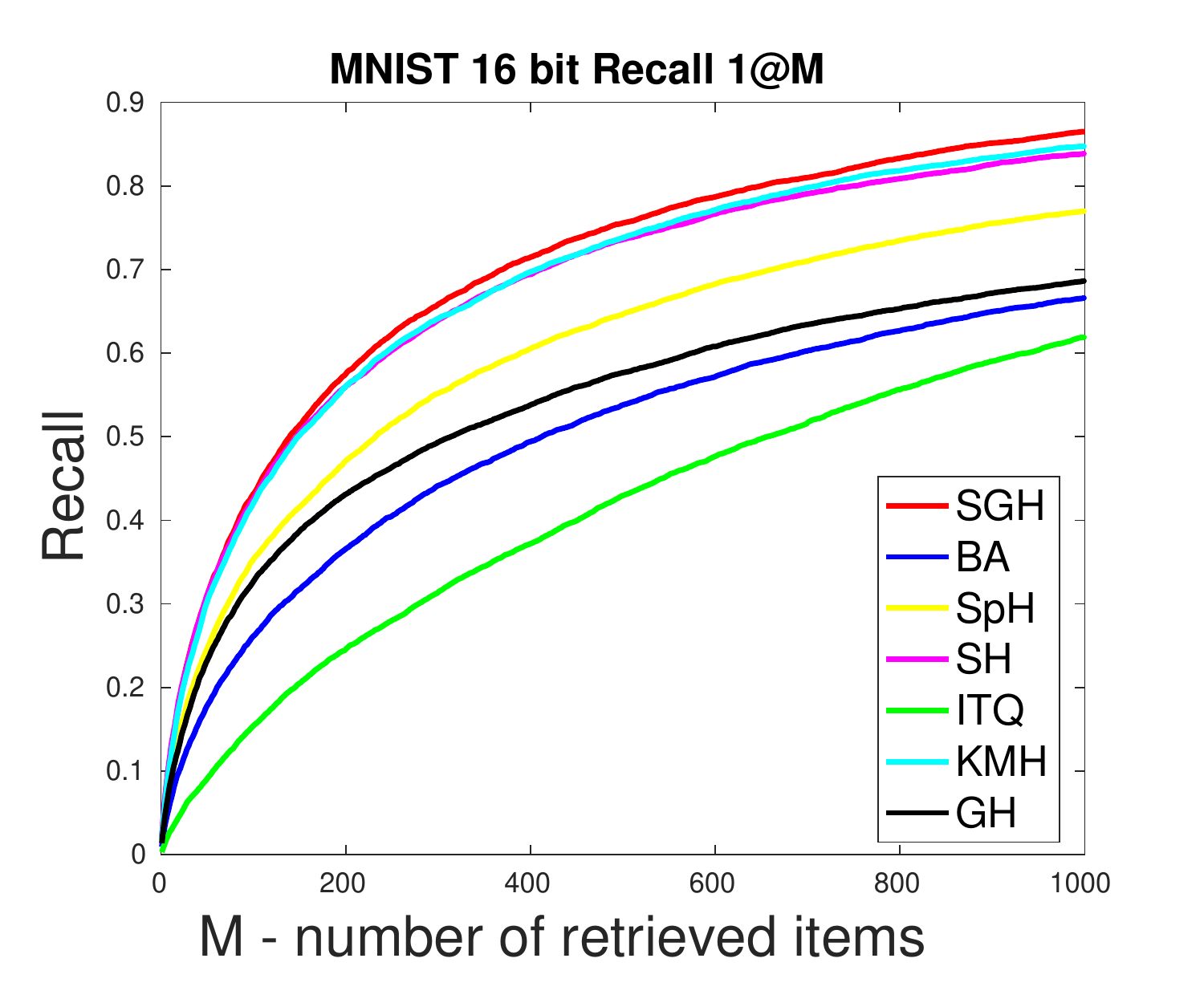}&\hspace{-5mm}
    \includegraphics[width=0.32\columnwidth, trim={0.25cm 0.6cm 1cm 1cm},clip]{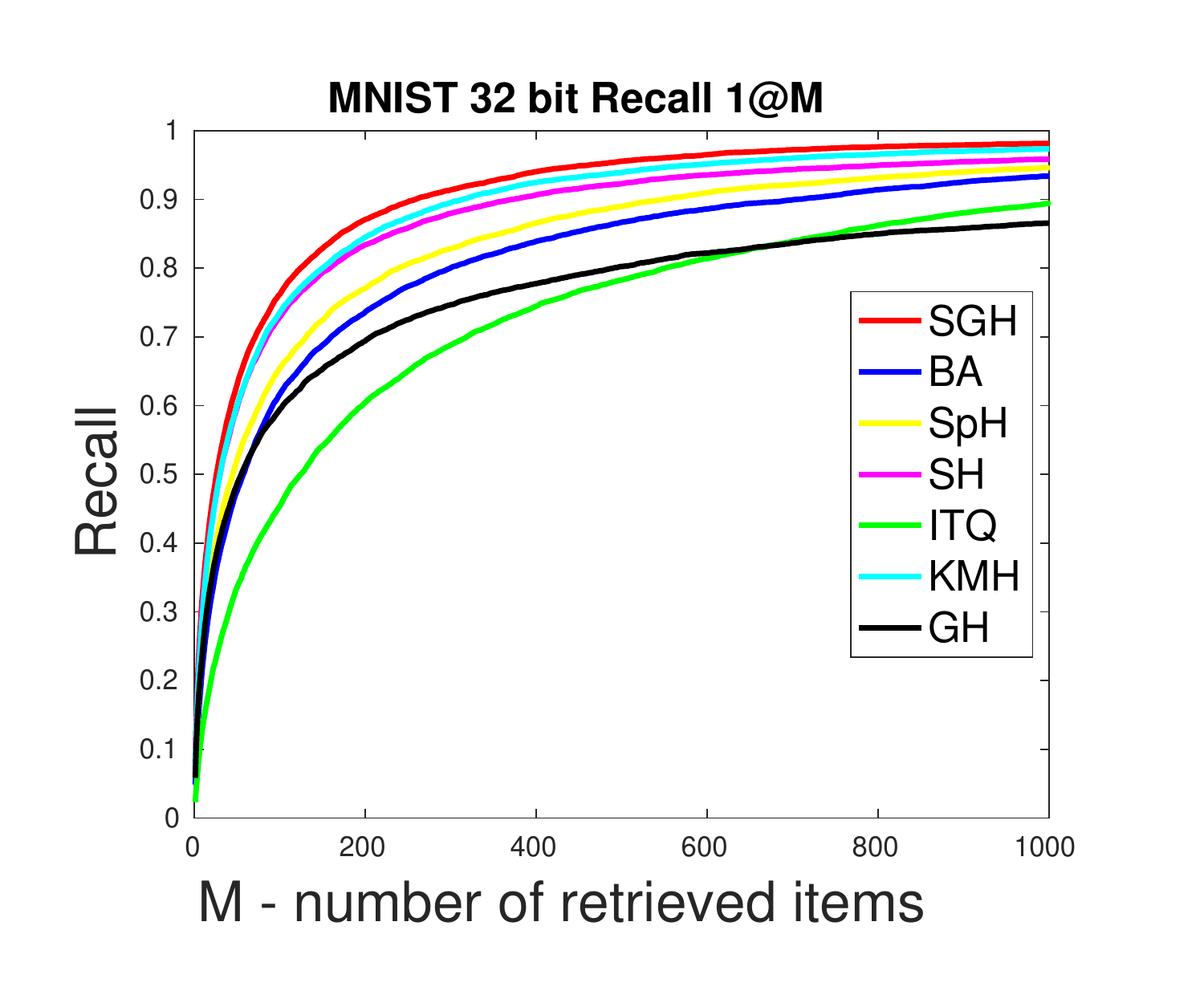}&\hspace{-5mm}
    \includegraphics[width=0.32\columnwidth, trim={0.25cm 0.6cm 1cm 1cm},clip]{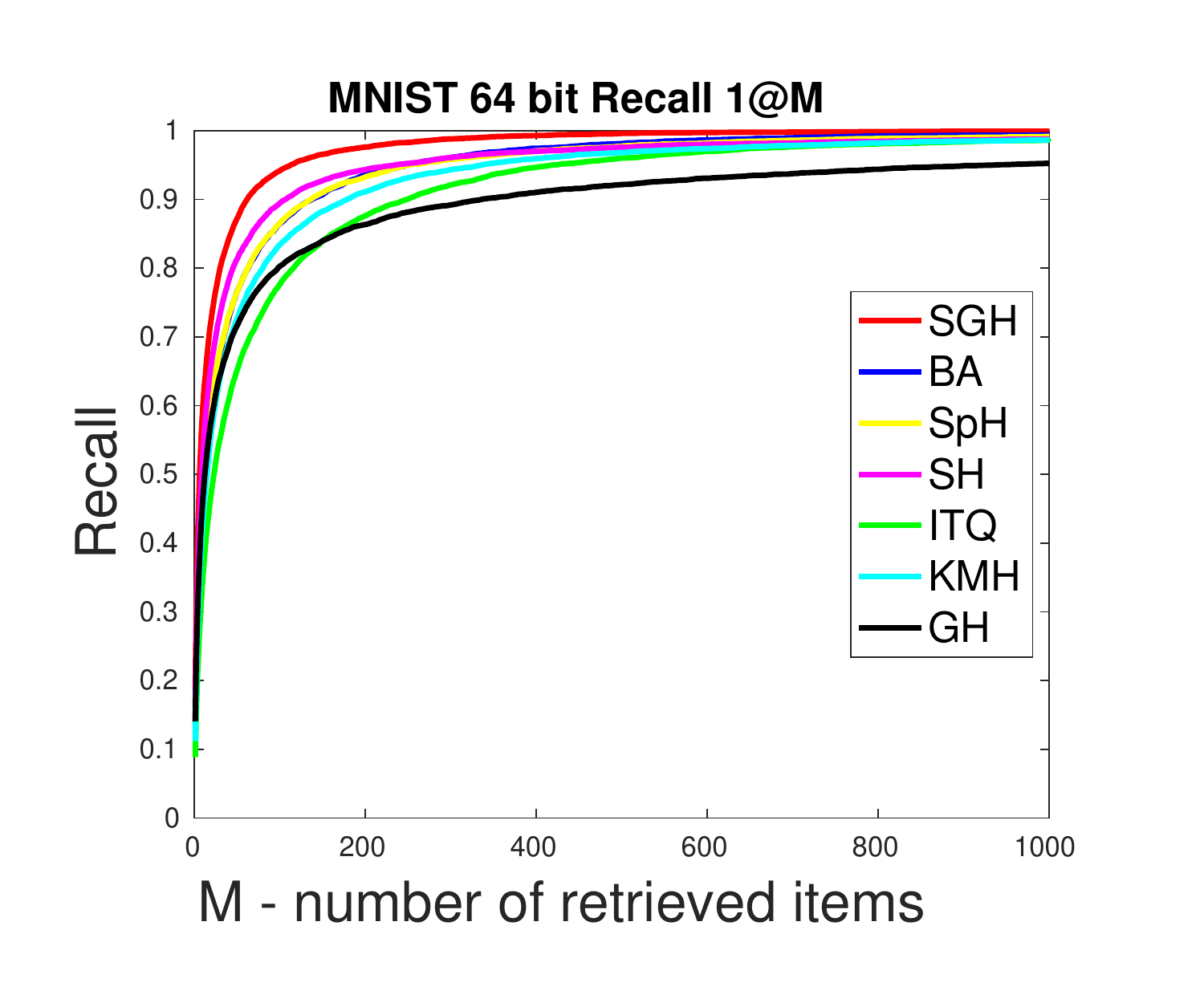}\\

    \includegraphics[width=0.33\columnwidth, trim={0.25cm 0.5cm 1cm 1cm},clip]{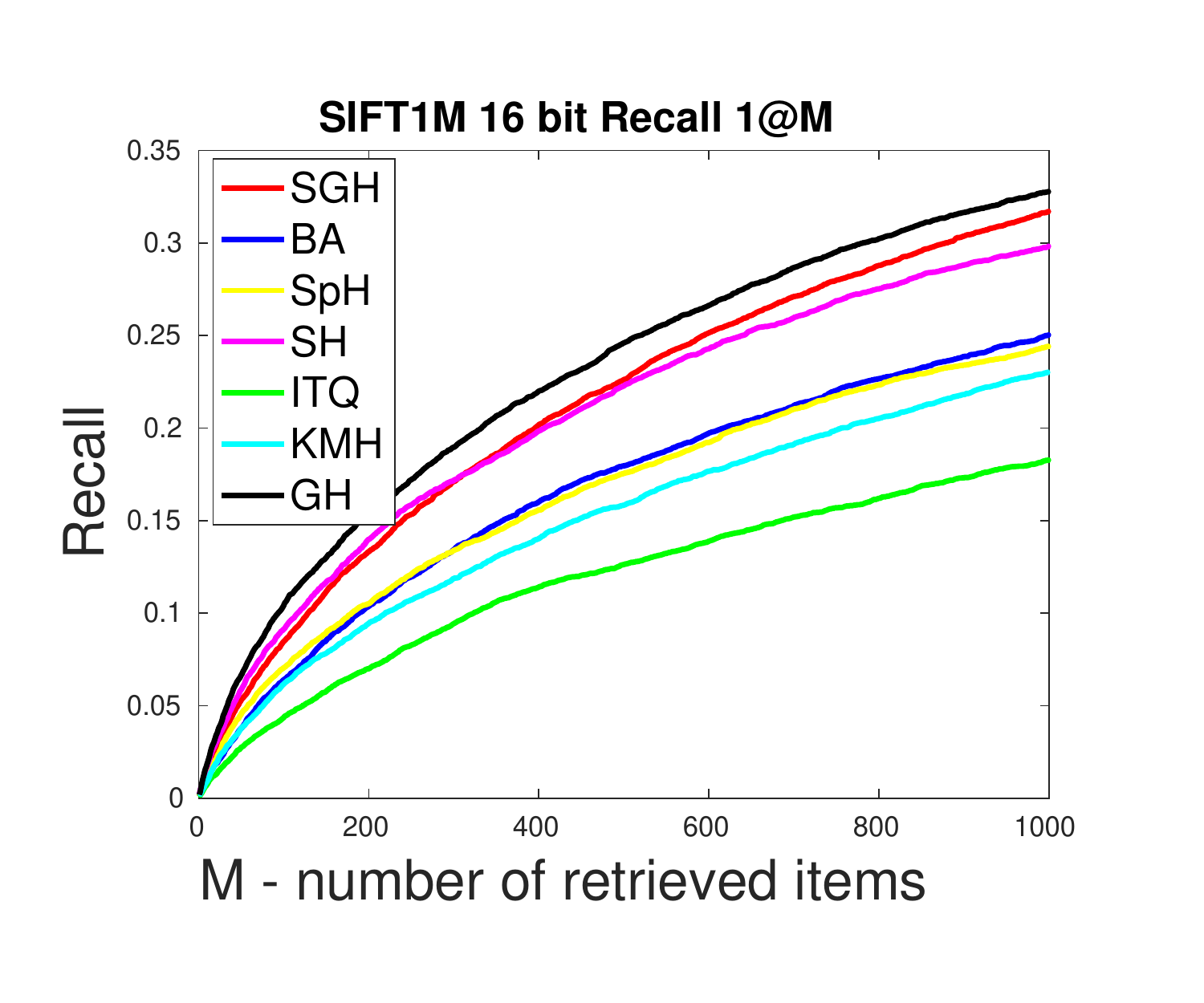}&\hspace{-5mm}
    \includegraphics[width=0.315\columnwidth, trim={0.25cm 0.3cm 1cm 0.5cm},clip]{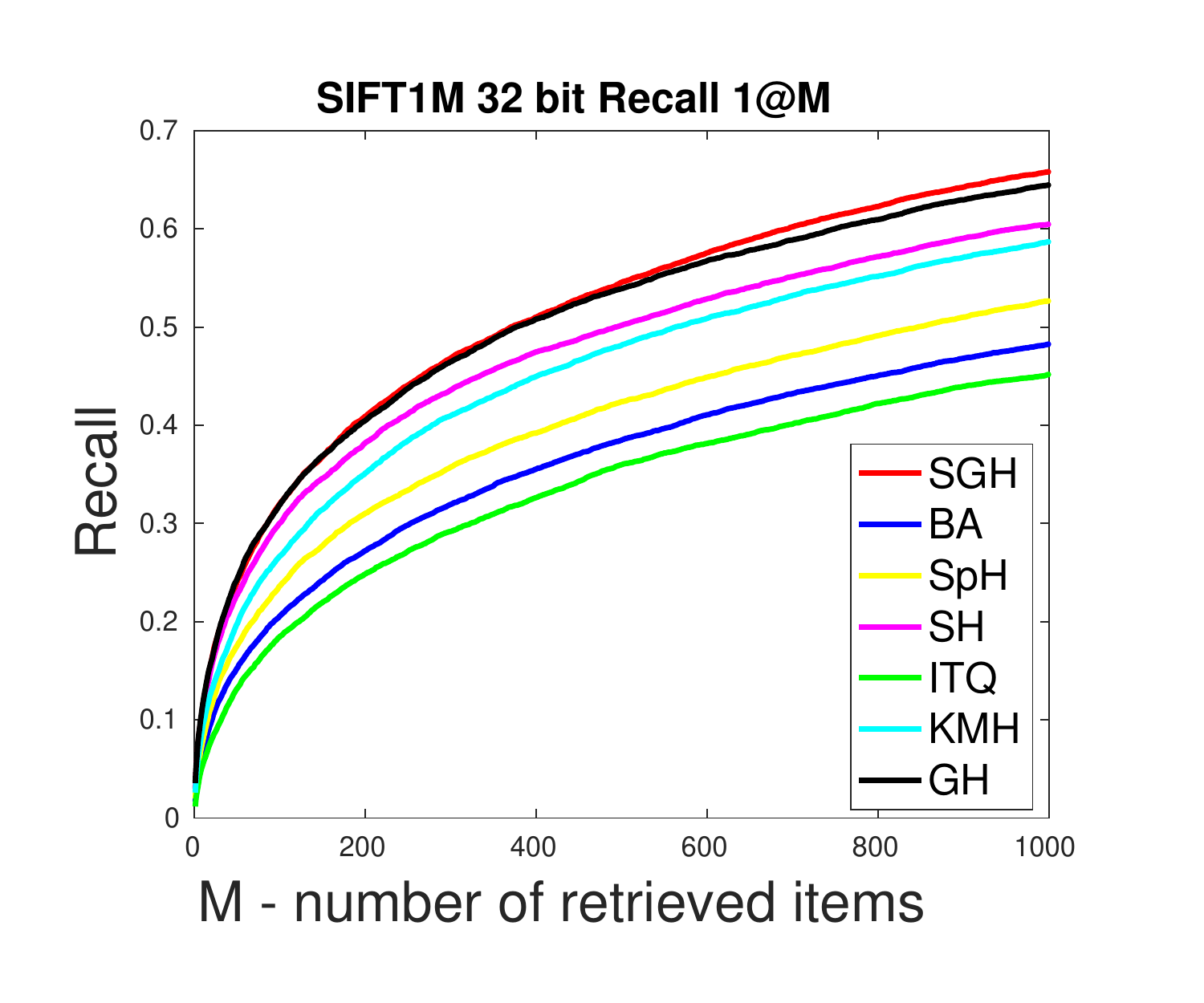}&\hspace{-5mm}
    \includegraphics[width=0.32\columnwidth, trim={0.25cm 0.3cm 1cm 0.5cm},clip]{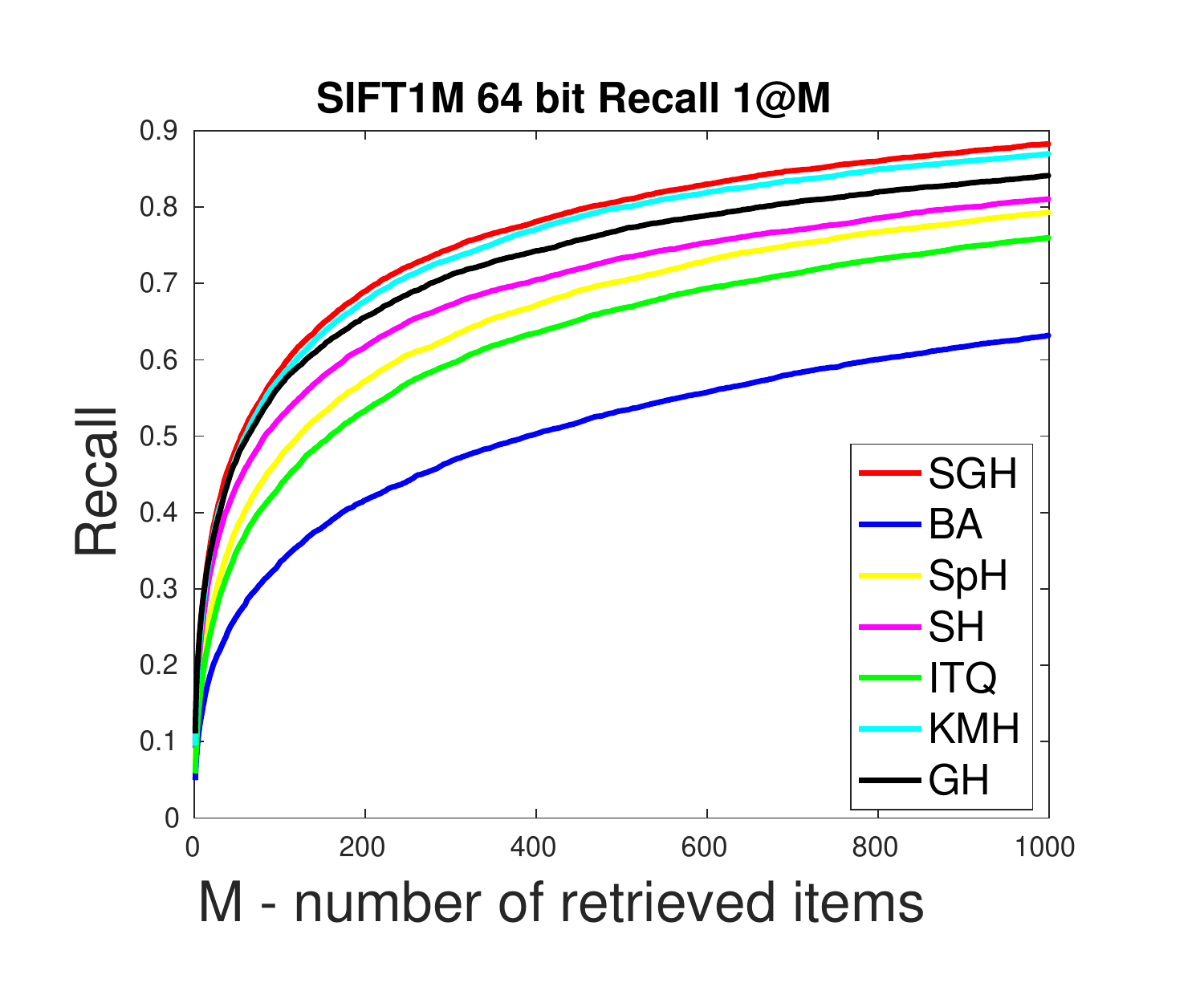}\\

    \includegraphics[width=0.32\columnwidth, trim={0.25cm 0.3cm 1cm 0.5cm},clip]{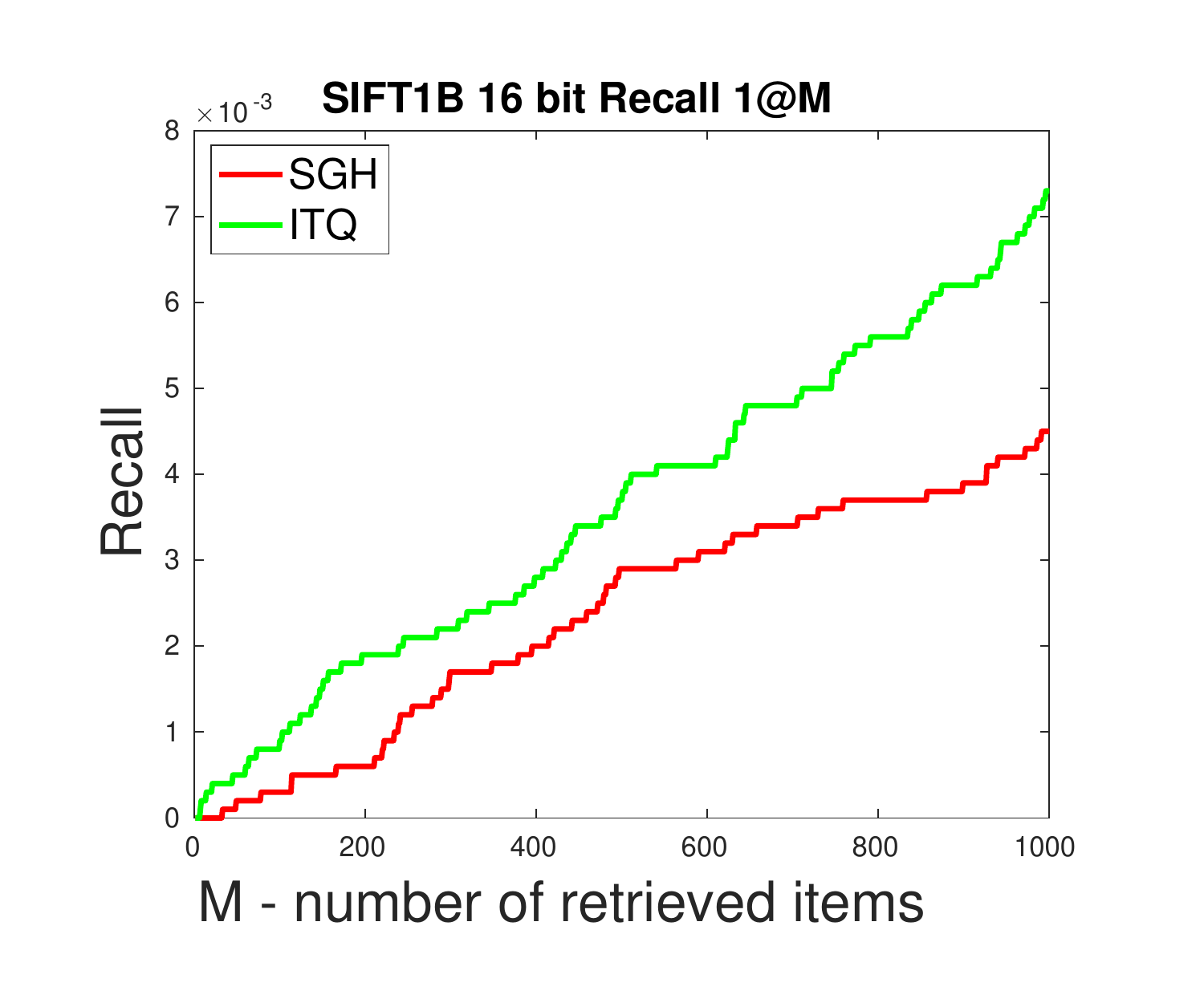}&\hspace{-5mm}
    \includegraphics[width=0.32\columnwidth, trim={0.25cm 0.3cm 1cm 0.5cm},clip]{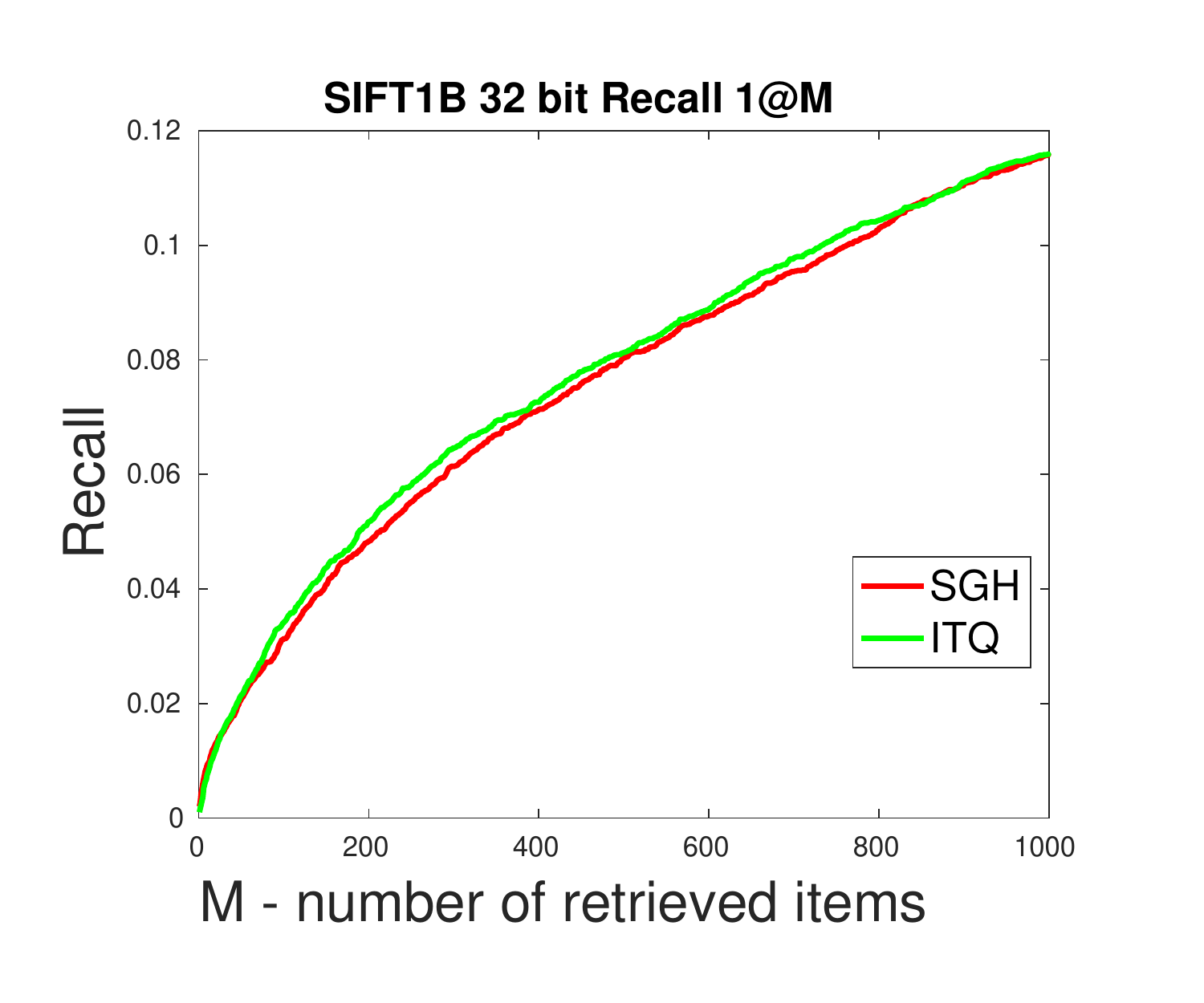}&\hspace{-5mm}
    \includegraphics[width=0.32\columnwidth, trim={0.25cm 0.3cm 1cm 0.5cm},clip]{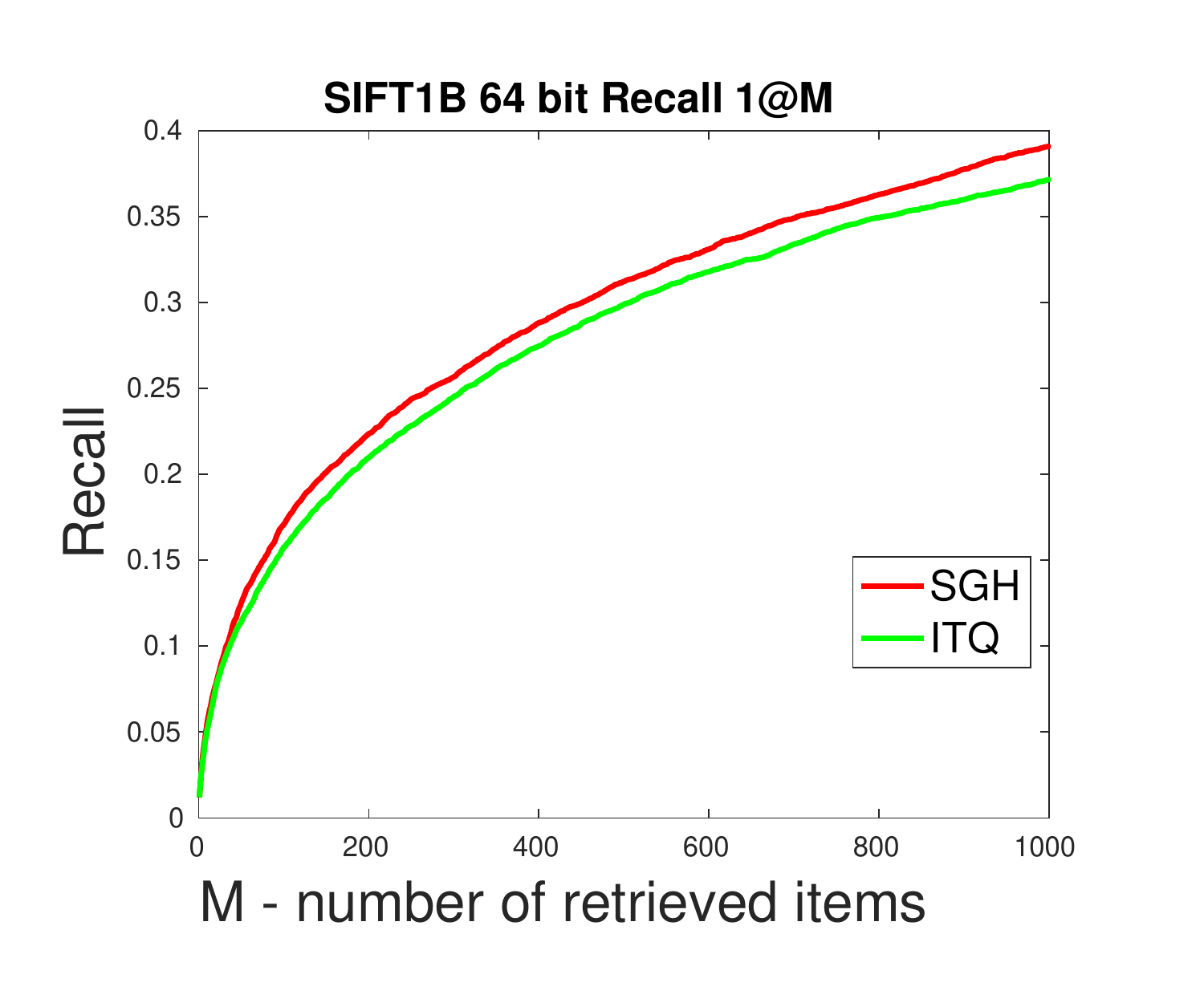}\\

	\includegraphics[width=0.32\columnwidth, trim={0.25cm 0.3cm 1cm 0.5cm},clip]{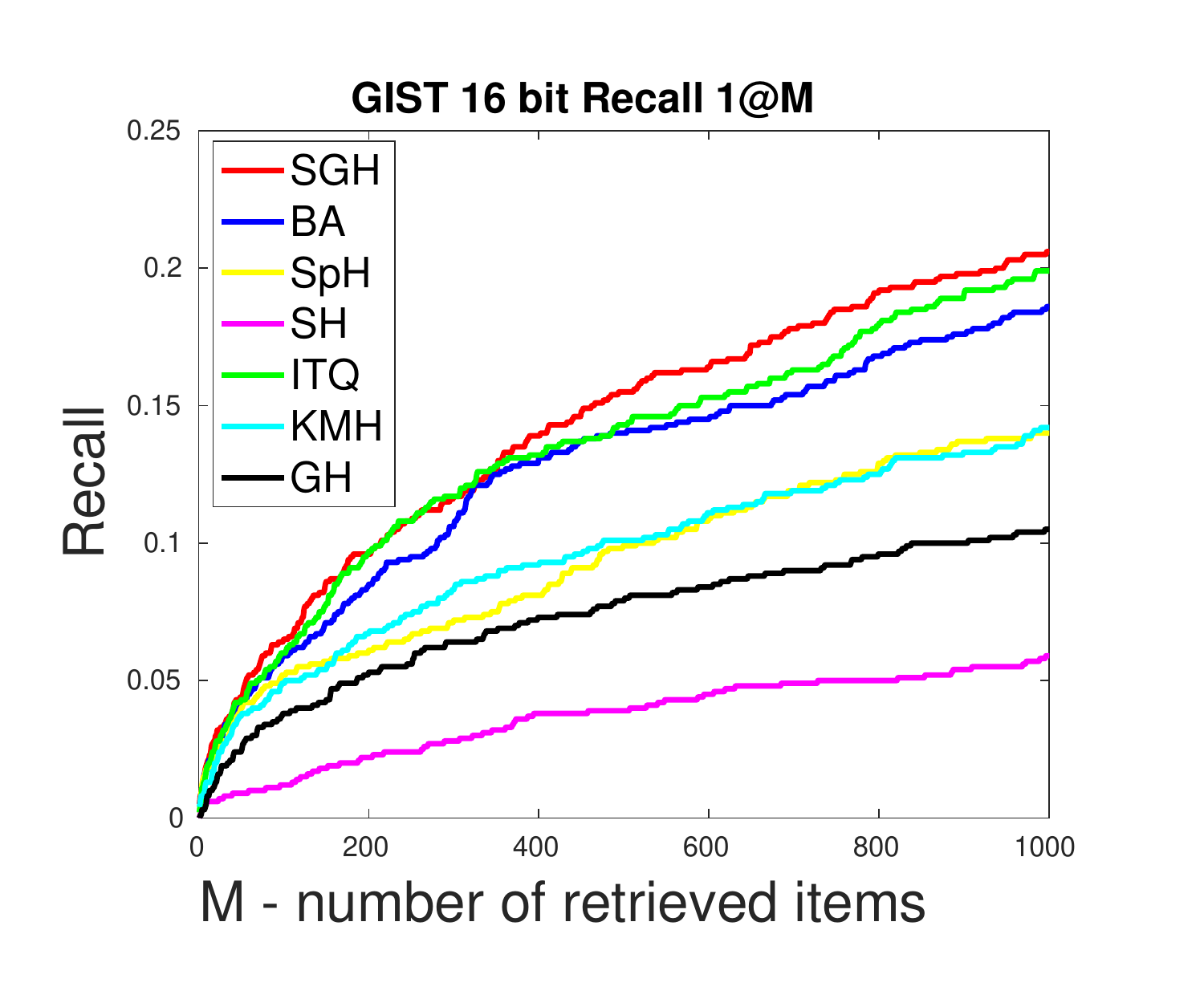}&\hspace{-5mm}
    \includegraphics[width=0.32\columnwidth, trim={0.25cm 0.5cm 1cm 0.5cm},clip]{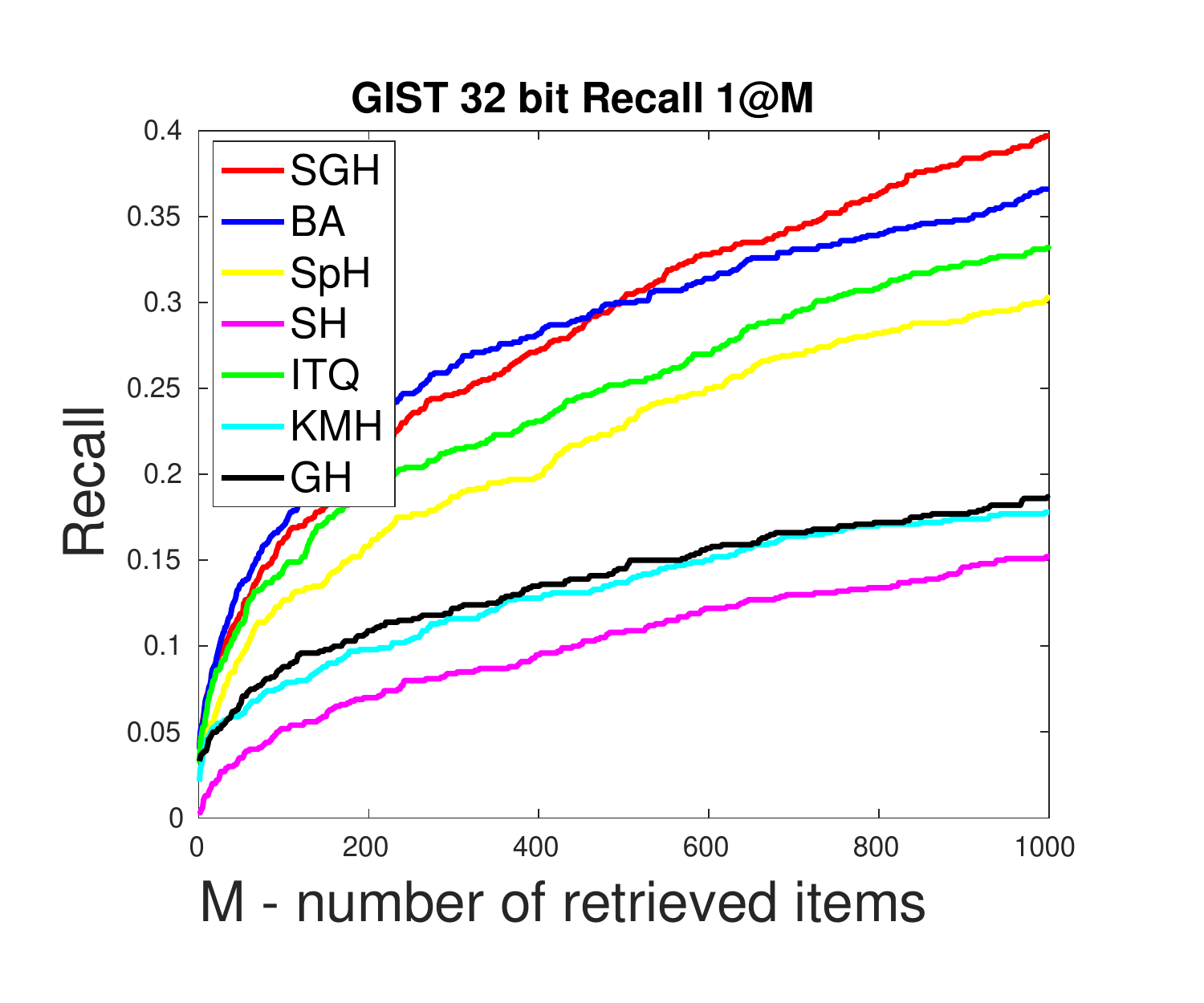}&\hspace{-5mm}
    \includegraphics[width=0.32\columnwidth, trim={0.25cm 0.5cm 1cm 0.5cm},clip]{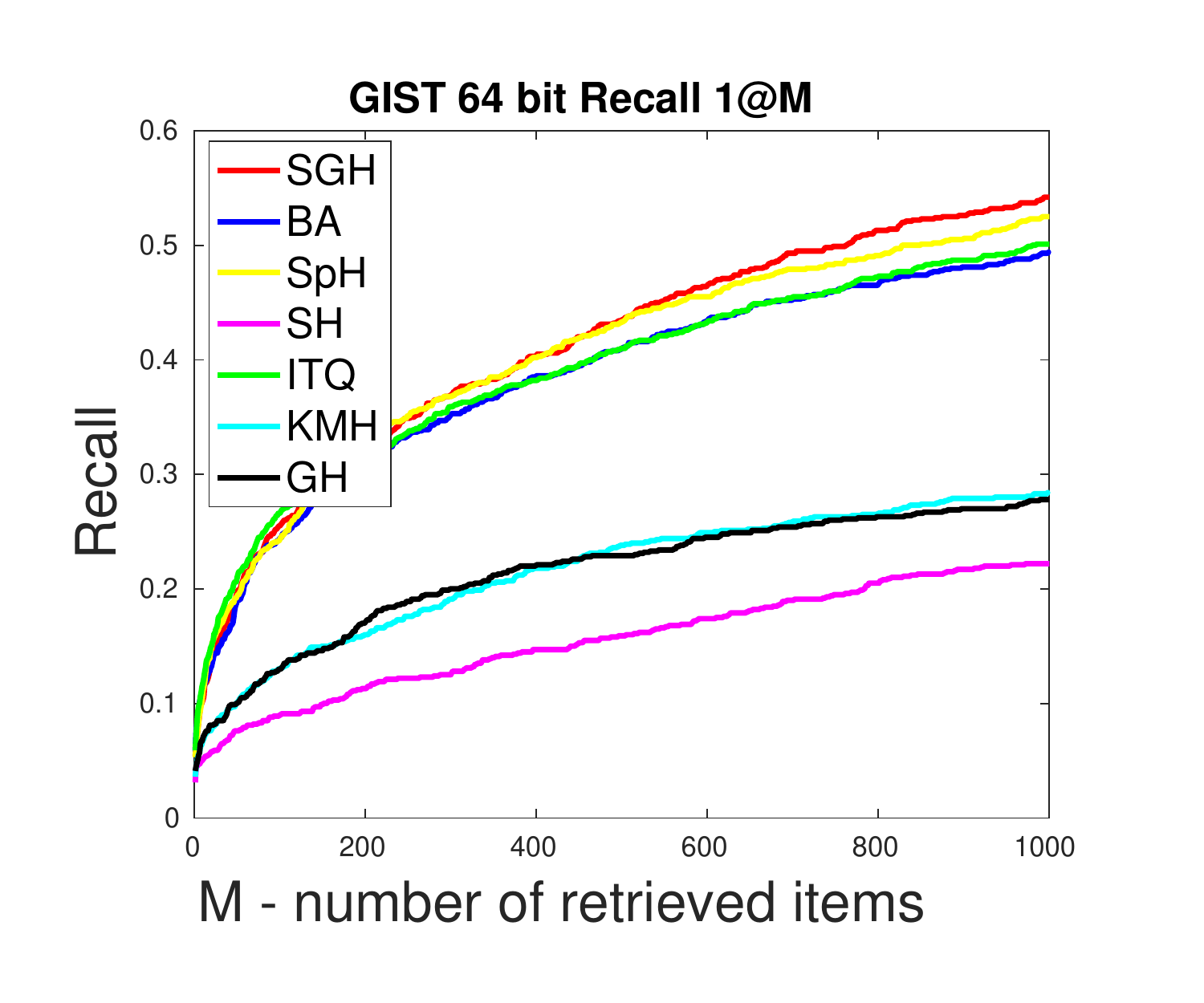}\\
  \end{tabular}
  \vspace{-4mm}
  \caption{L2NNS comparison on \texttt{MNIST}, \texttt{SIFT-1M}, \texttt{SIFT-1B}, and \texttt{GIST-1M} with the length of binary bits from $16$ to $64$. We evaluate the performance with Recall 1@$M$, where $M$ increasing to $1000$. }
  \label{fig:recall_1}
\end{center}
\end{figure*}
We shows the training time comparison between BA and SGH on \texttt{MNIST} and \texttt{GIST-1M} in Figure~\ref{fig:more_train_time}. The results are similar to the performance on \texttt{SIFT-1M}. The proposed distributional SGD learns the model much faster.

\subsection{More Evaluation on L2NNS Retrieval Tasks}

We also use different RecallK@N to evaluate the performances of our algorithm and the competitors. We first evaluated the performance of the algorithms with Recall 1@N in Figure~\ref{fig:recall_1}. This is an easier task comparing to $K=10$. Under such measure, the proposed SGH still achieves the state-of-the-art performance.

In Figure~\ref{fig:recall_100_bits}, we set $K, N=100$ and plot the recall by varying the length of the bits on \texttt{MNIST}, \texttt{SIFT-1M}, and \texttt{GIST-1M}. This is to show the effects of length of bits in different baselines. Similar to the Recall10@N, the proposed algorithm still consistently achieves the state-of-the-art performance under such evaluation measure. 
\begin{figure}[h]
\hspace{-8mm}
\begin{center}
  \begin{tabular}{ccc}
  	\includegraphics[width=0.31\columnwidth, trim={0.25cm 0.3cm 1cm 0.5cm},clip]{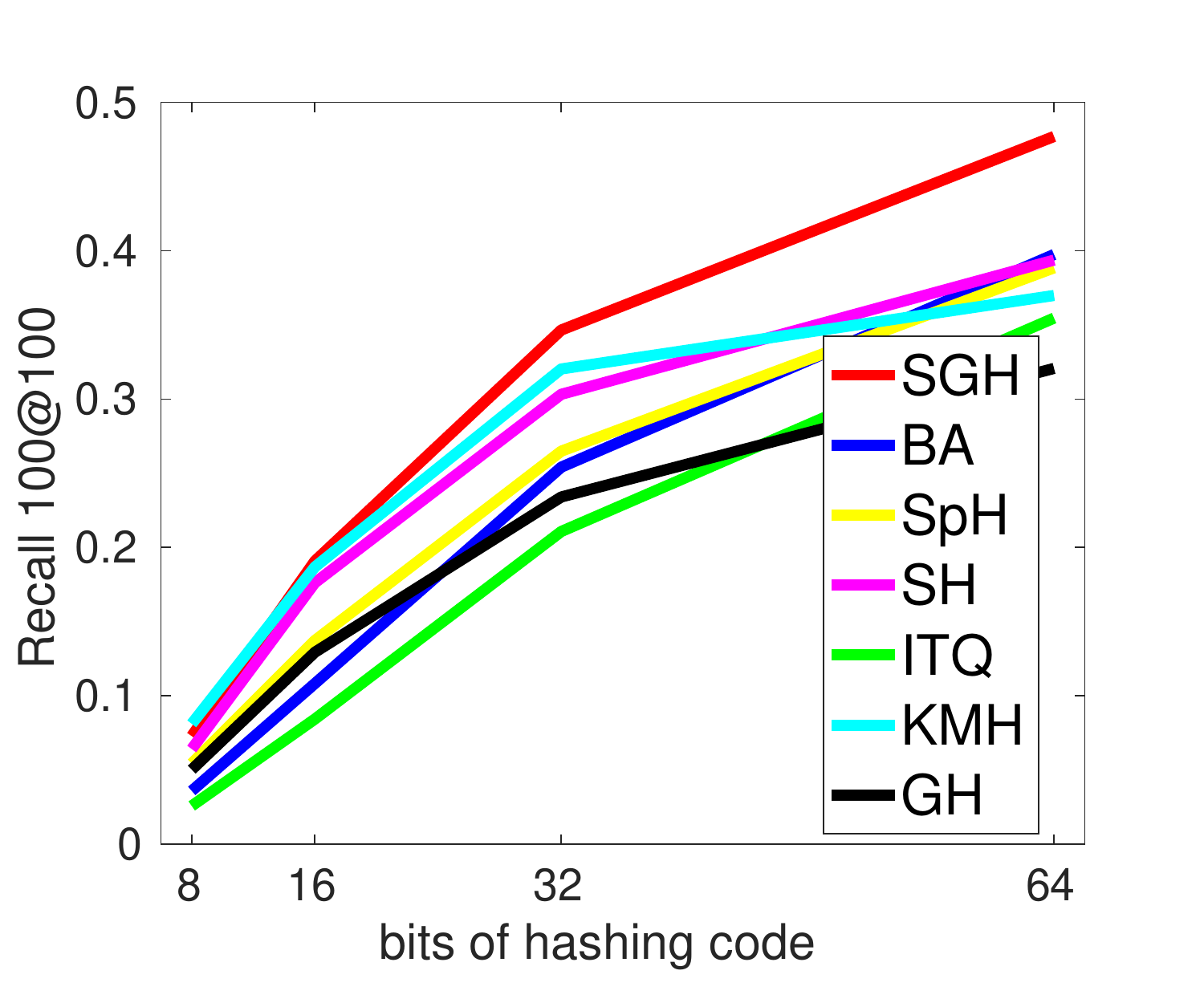}&\hspace{-5mm}
    \includegraphics[width=0.31\columnwidth, trim={0.225cm 0.3cm 1cm 0.5cm},clip]{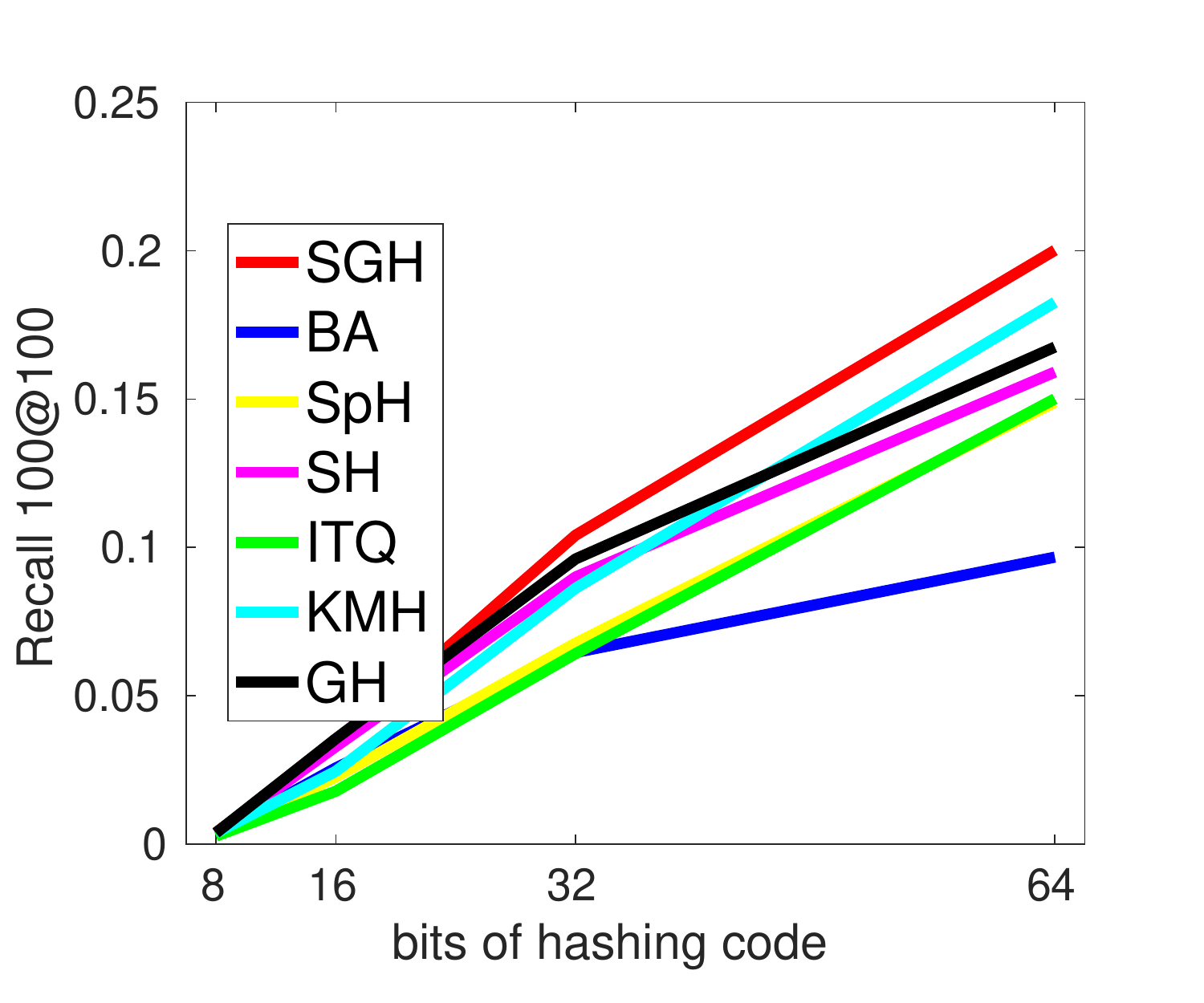}&\hspace{-5mm}
    \includegraphics[width=0.31\columnwidth, trim={0.25cm 0.3cm 1cm 0.5cm},clip]{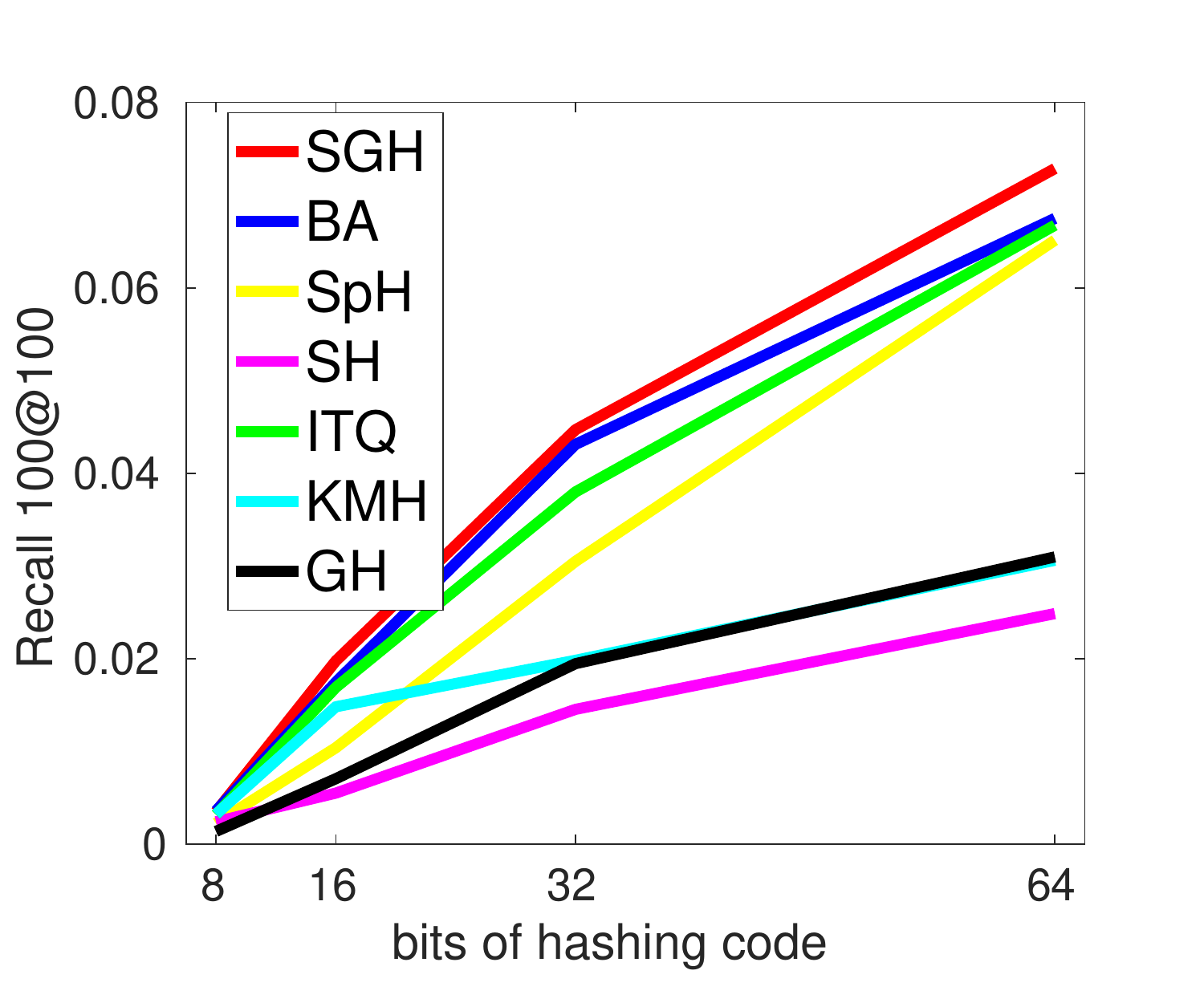}\\
    (a) L2NNS on \texttt{MNIST}  & (b) L2NNS on \texttt{SIFT-1M} & (c) L2NNS on \texttt{GIST-1M}\\
  \end{tabular}
  \caption{L2NNS comparison on \texttt{MNIST}, \texttt{SIFT-1M}, and \texttt{GIST-1M} with Recall 100@100 for the length of bits from $8$ to $64$.}
  \label{fig:recall_100_bits}
\end{center}
\end{figure}

\section{Stochastic Generative Hashing For Maximum Inner Product Search}\label{appendix:sgh_mips}

In Maximum Inner Product Search~(MIPS) problem, we evaluate the similarity in terms of inner product which can avoid the scaling issue, \ie, the length of the samples in reference dataset and the queries may vary. The proposed model can also be applied to the MIPS problem. In fact, the Gaussian reconstruction model also preserve the inner product neighborhoods. Denote the asymmetric inner product as $x^\top U h_y$, we claim 
\vspace{-3mm}
\begin{proposition}\label{prop:mips_preserve}
The Gaussian reconstruction error is a surrogate for asymmetric inner product preservation.
\end{proposition}
\vspace{-2mm}
\begin{proof} 
We evaluate the difference between inner product and the asymmetric inner product,
\begin{eqnarray*}
\|x^\top y - x^\top U^\top h_y\|_2 = \|x^\top \rbr{y - U^\top h_y}\|_2\le \|x\|_2\|y - U^\top h_y\|_2,
\end{eqnarray*}
which means minimizing the Gaussian reconstruction, \ie, $-\log p(x|h)$, error will also lead to asymmetric inner product preservation.
\vspace{-1mm}
\end{proof}
We emphasize that our method is designed for hashing problems primarily. Although it can be used for MIPS problem, it is different from the product quantization and its variants whose distance are calculated based on lookup table. The proposed distributional SGD can be extended to quantization. This is out of the scope of this paper, and we will leave it as the future work.

\subsection{MIPS Retrieval Comparison}
To evaluate the performance of the proposed SGH on MIPS problem, we tested the algorithm on \texttt{WORD2VEC} dataset for MIPS task. Besides the hashing baselines, since KMH is the Hamming distance generalization of PQ, we replace the KMH with product quantization~\cite{pq}. We trained the SGH with 71,291 samples and evaluated the performance with 10,000 query. Similarly, we vary the length of binary codes from $16$, $32$ to $64$, and evaluate the performance by Recall 10@N. We calculated the ground-truth via retrieval through the original inner product. The performances are illustrated in Figure~\ref{fig:MIPS_bits}. The proposed algorithm outperforms the competitors significantly, demonstrating the proposed SGH is also applicable to MIPS task.
\begin{figure}[t]
\begin{center}
\hspace{-8mm}
\begin{tabular}{ccc}
    \includegraphics[width=0.32\columnwidth, trim={0.25cm 0.3cm 1cm 0.5cm},clip]{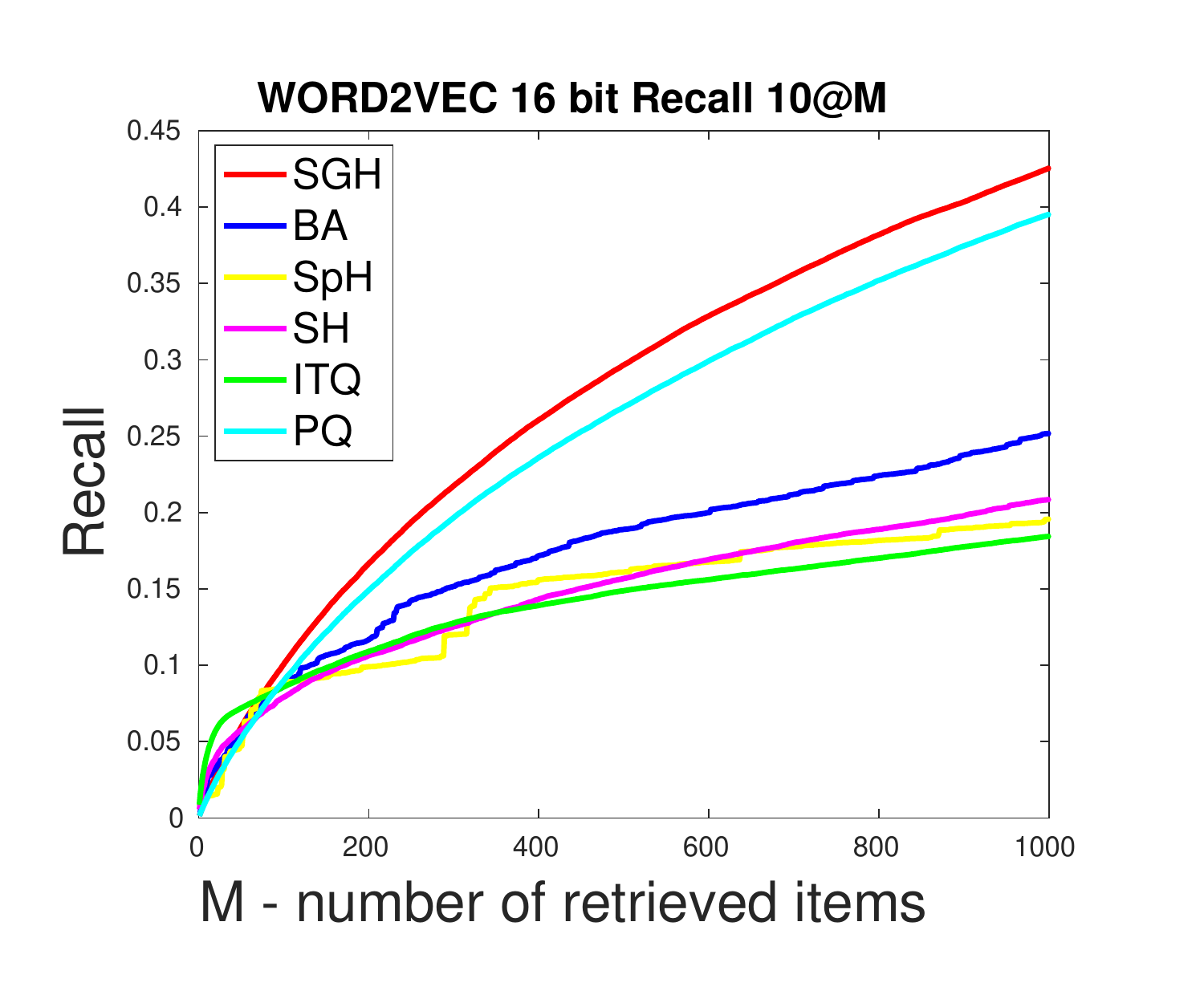}&\hspace{-5mm}
    \includegraphics[width=0.32\columnwidth, trim={0.25cm 0.3cm 1cm 0.5cm},clip]{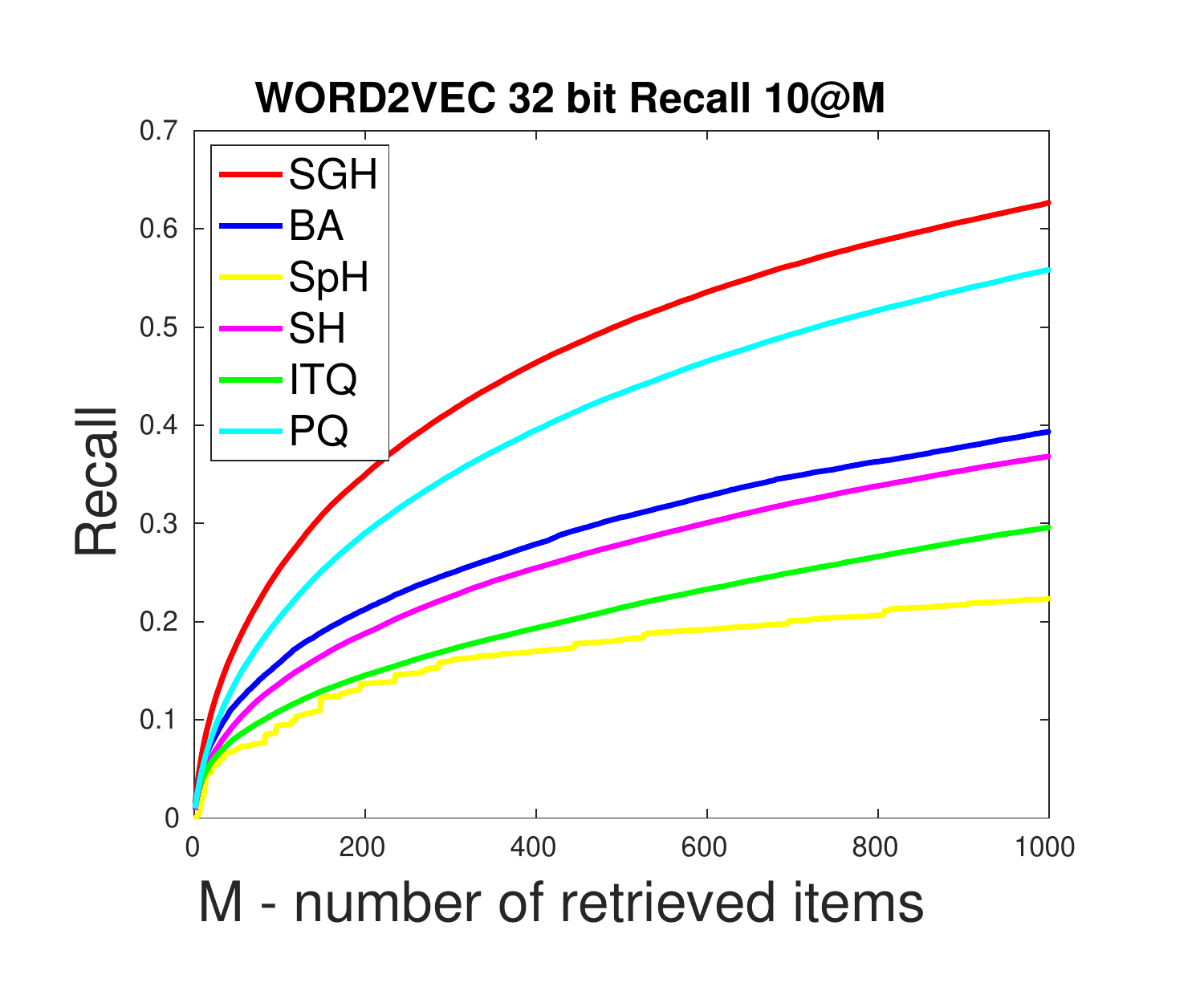}&\hspace{-5mm}
    \includegraphics[width=0.29\columnwidth, trim={0.25cm 0cm 0.7cm 0.5cm},clip]{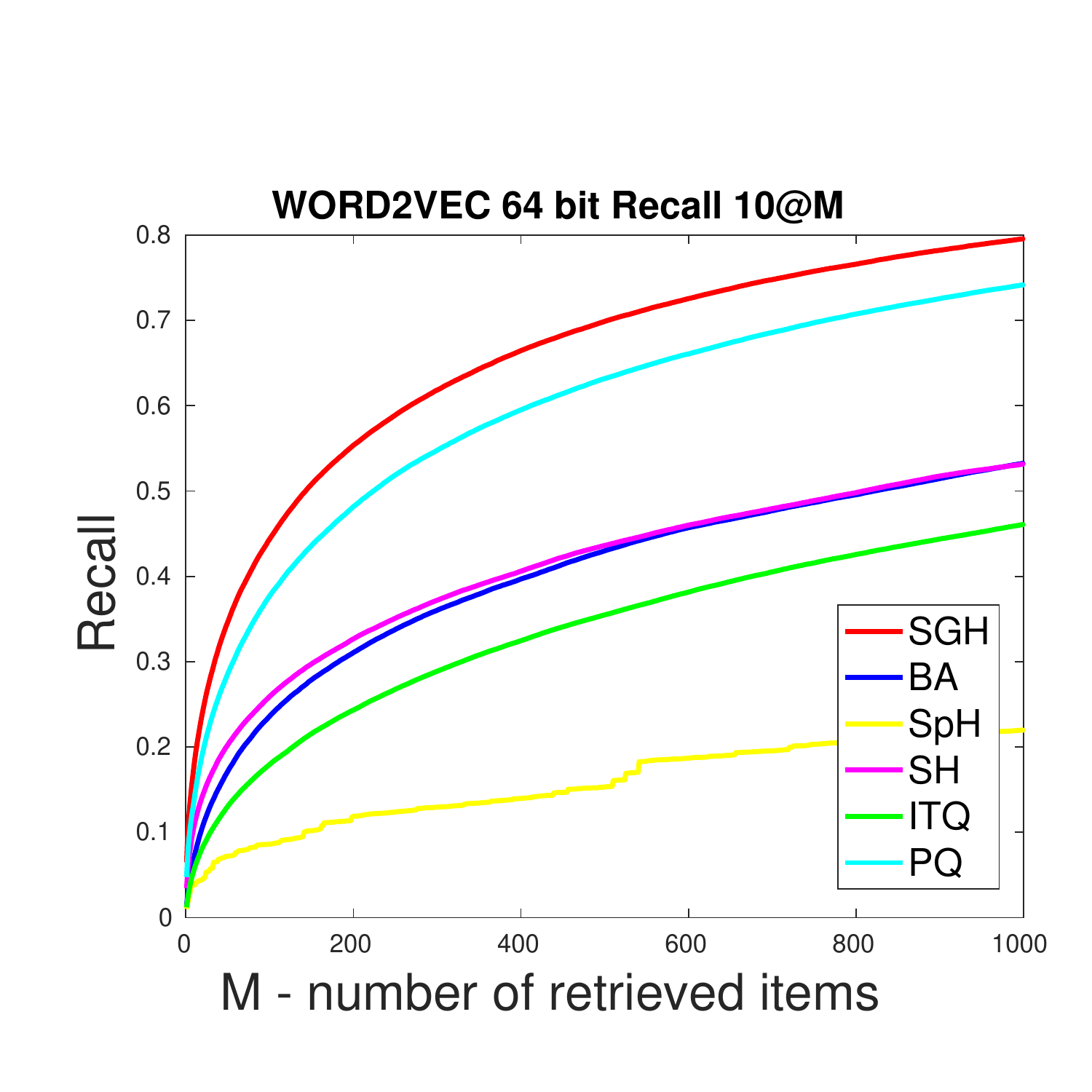}\\
  \end{tabular}
  \caption{MIPS comparison on \texttt{WORD2VEC} with the length of binary bits from $16$ to $64$. We evaluate the performance with Recall 10@$M$, where $M$ increasing to $1000$. }
  \label{fig:MIPS_bits}
\end{center}
\end{figure}
%

\section{Generalization}\label{appendix:generalization}

We generalize the basic model to translation and scale invariant extension, semi-supervised extension, as well as coding with $h\in \{-1, 1\}^l$. 

\subsection{Translation and Scale Invariant Reduced-MRFs}

As we known, the data may not zero-mean, and the scale of each sample in dataset can be totally different. To eliminate the translation and scale effects, we extend the basic model to translation and scale invariant reduced-MRFs by introducing parameter $\alpha$ to separate the translation effect and the latent variable $z$ to model the scale effect in each sample $x$, therefore, the potential function becomes
\begin{eqnarray}\label{eq:scale_free_mrf}
E(x, h, z) = -\beta^\top h + \frac{1}{2\rho^2}{(x -\alpha - U^\top(z\cdot h))^\top (x -\alpha - U^\top(z\cdot h))},
\end{eqnarray}
where $\cdot$ denotes element-wise product, $\alpha\in \RR^d$ and $z\in \RR^l$. Comparing to~\eq{eqn:reduced_mrf}, we replace $U^\top h$ with $U^\top(z\cdot h) + \alpha$ so that the translation and scale effects in both dimension and sample are modeled explicitly.

We treat the $\alpha$ as parameters and $z$ as latent variable. Assume the independence in posterior for computational efficiency, we approximate the posterior $p(z, h|x)$ with $q(h|x; W_h)q(z|x; W_z)$, where $W_h, W_z$ denotes the parameters in the posterior approximation. With similar derivation, we obtain the learning objective as 
\begin{equation}\label{eq:scale_free_opt}
\max_{U, \alpha, \beta, \rho; W_h, W_z} \frac{1}{N}\sum_{i=1}^N \EE_{q(h|x_i)q(z|x_i)}\sbr{-E(x, h, z) - \log q(h|x_i) - \log q(z|x_i)}.
\end{equation}
Obviously, the proposed distributional SGD is still applicable to this optimization.

\subsection{Semi-supervised Extension}
Although we only focus on learning the hash function in unsupervised setting, the proposed model can be easily extended to exploit the supervision information by introducing pairwise model, \eg, \cite{ZhaZhaLiGuo14,ZhuLonWanCao16}. Specifically, we are provided the (partial) supervision information for some pairs of data, \ie, $\Scal = \{x_i, x_i, y_{ij}\}_{i,j}^M$, where
$$
  y_{ij} = \begin{cases}
    1       & \quad \text{if } x_i \in \Ncal\Ncal(x_j)\text{ or } x_j \in \Ncal\Ncal(x_i)\\
    0  & \quad \text{o.w.} \\
  \end{cases},
$$
and $\Ncal\Ncal(x)$ stands for the set of nearest neighbors of $x$. Besides the original Gaussian reconstruction model in the basic model in~\eq{eqn:reduced_mrf}, we introduce the pairwise model $p(y_{ij}|h_i, h_j) = \Bcal(\sigma(h_i^\top h_j))$ into the framework, which results the joint distribution over $x, y, h$ as 
\begin{eqnarray*}
p(x_i, x_j, h_i, h_j, y_{ij})= 
p(x_i|h_i)p(x_j|h_j)p(h_i)p(h_j)p(y_{ij}|h_i, h_j)^{\one_{\Scal}(ij)},
\end{eqnarray*}
where $\one_{\Scal}(ij)$ is an indicator that outputs $1$ when $(x_i, x_j)\in \Scal$, otherwise $0$. Plug the extended model into the Helmholtz free energy, we have the learning objective as,
\begin{eqnarray*}
\max_{U, \beta, \rho; W}\frac{1}{N^2}\sum_{i, j=1}^{N^2} &&\Big(\EE_{q(h_i|x_i)q(h_j|x_j)}\sbr{\log p(x_i, x_j, h_i, h_j)} + \EE_{q(h_i|x_i)q(h_j|x_j)}\sbr{\one_{\Scal}(ij)\log p(y_{ij}|h_i, h_j)}\nonumber \\
&&- \EE_{q(h_i|x_i)q(h_j|x_i)}\sbr{\log q(h_j|x_j)q(h_j|x_i)}\Big), \nonumber
\end{eqnarray*}
Obviously, the proposed distributional SGD is still applicable to the semi-supervised extension.

\subsection{$\{\pm 1\}$-Binary Coding}

In the main text, we mainly focus on coding with $\{0, 1\}$. In fact, the proposed model is applicable to coding with $\{-1, 1\}$ with minor modification. Moreover, the proposed distributional SGD is still applicable. We only discuss the basic model here, the model can also be extended to scale-invariant and semi-supervised variants. 

If we set $h\in \{-1, 1\}^l$, the potential function of basic reduced-MRFs~\eq{eqn:reduced_mrf} does not have any change, \ie, 
\begin{eqnarray}\label{eq:pm_reduced_mrf}
E(x, h) =  - \beta^\top h + \frac{1}{2\rho^2}\rbr{x^\top x + h^\top U^\top  Uh -  2x^\top Uh}.
\end{eqnarray}
We need to modify the parametrization of $q(h|x)$ as
\begin{eqnarray}\label{eq:pm_encoder_param}
q(h|x) = \prod_{i=1}^l \sigma(w_i^\top x)^{\frac{1 + h_i}{2}}\rbr{1 - \sigma(w_i^\top x)}^{\frac{1 - h_i}{2}}.
\end{eqnarray}
Therefore, the stochastic neuron becomes
\begin{eqnarray*}
f(z, \xi):= \begin{cases}
    1       & \quad \text{if } \sigma(z) \ge \xi  \\
    -1  & \quad \text{if } \sigma(z) < \xi
\end{cases}.
\end{eqnarray*}
With similar derivation, we have the distributional derivative of the objective w.r.t. $W$ as 
\begin{eqnarray}\label{eq:pm_new_grad}
\nabla_{W}L_{sn} = \EE_{\xi}\sbr{\Delta_f \ell(f(z, \xi))\nabla_z \sigma(z) x^\top},
\end{eqnarray}
where $\sbr{\Delta_f \ell(f(z, \xi))}_k = \ell(f^1_{k}) - \ell(f^{-1}_{k})$. Furthermore, we have a similar biased gradient estimator as
\begin{eqnarray}\label{eq:pm_new_grad}
\tilde\nabla_{W}L_{sn} = \EE_{\xi}\sbr{\nabla_f \ell(f(z, \xi))\nabla_z \sigma(z) x^\top}.
\end{eqnarray}
Plug these modification into the model and algorithm, we can learn a $\{-1, 1\}$-encoding function.

\end{appendix}

\end{document}